%% file: tnnls.tex
\documentclass[lettersize,journal]{IEEEtran}

\usepackage[colorlinks,urlcolor=blue,linkcolor=blue,citecolor=blue]{hyperref}
\usepackage{color,array}
\usepackage{xcolor}
\usepackage{graphicx}
\usepackage{amsfonts}
\usepackage{amsmath}
\usepackage{amssymb}
\usepackage{gensymb}
\usepackage{algorithm}
\usepackage{algorithmicx}
\usepackage{algpseudocode}
\usepackage{amsthm}
\usepackage{cite}
\usepackage{booktabs}
\usepackage{adjustbox}
\usepackage{todonotes}
\usepackage{subfigure}
\usepackage{wrapfig}
\usepackage{tikz-cd}

\newtheorem{definition}{Definition}
\newtheorem{theorem}{Theorem}
\newtheorem{lemma}{Lemma}
\newtheorem{step}{Step}
\newtheorem{assumption}{Assumption}
\newtheorem{proposition}{Proposition}
\newcommand{\initdist}{f}

\include{def}

\begin{document}

	\title{Learn Zero-Constraint-Violation Policy in Model-Free Constrained Reinforcement Learning}

	\author{Haitong Ma, Changliu Liu, Shengbo Eben Li, Sifa Zheng, Wenchao Sun, and Jianyu Chen
		\thanks{Haitong Ma, Shengbo Eben Li, and Sifa Zheng are with the State Key Laboratory of Automotive Safety and Energy, School of Vehicle and Mobility, Tsinghua University, Beijing 100084, China, and also with the Center for Intelligent Connected Vehicles and Transportation, Tsinghua University, Beijing 100084, China (e-mail: \{maht19@mails, lisb04@, zsf@\}tsinghua.edu.cn).}
		\thanks{Changliu Liu are with the Robotics Institute, Carnegie Mellon University, Pittsburgh,  PA 15213 USA (e-mail: cliu6@andrew.cmu.edu).}
		\thanks{Jianyu Chen are with the Institute of Interdisciplinary Information Science, Tsinghua University, Beijing 100084, China, and also with Shanghai Qizhi Insititute, Shanghai 200232, China (e-mail: jianyuchen@tsinghua.edu.cn).}
		\thanks{All correspondence should be sent to Shengbo Eben Li.}}
	
\markboth{IEEE Transactions on Neural Networks and Learning Systems}
	{Haitong Ma \MakeLowercase{\textit{et al.}}: Learn Zero Constraint Violation Policy in Model-Free Constrained Reinforcement Learning}
	
% 	\IEEEpubid{0000--0000/00\$00.00~\copyright~2021 IEEE}
	
	\maketitle
	
	\begin{abstract}
		In the trial-and-error mechanism of reinforcement learning (RL), a notorious contradiction arises when we expect to learn a \emph{safe} policy: how to learn a safe policy without enough data and prior model about the dangerous region? 
		%% without enough data and prior model about the dangerous region. 
		Existing methods mostly use the posterior penalty for dangerous actions, which means that the agent is not penalized until experiencing danger. This fact causes that the agent cannot learn a zero-violation policy \emph{even after convergence}. Otherwise, it would not receive any penalty and lose the knowledge about danger. In this paper, we propose the safe set actor-critic (SSAC) algorithm, which confines the policy update using safety-oriented energy functions, or the \emph{safety indexes}. The safety index is designed to increase rapidly for potentially dangerous actions, which allow us to locate the \emph{safe set} on the action space, or the \emph{control safe set}. Therefore, we can identify the dangerous actions \emph{prior to} taking them, and further obtain a zero constraint-violation policy after convergence.
		%%People might not know what is safe set algorithm, or energy function safety certificate. Try to use a more readable way to introduce your insight here.
		%%obtains a zero constraint-violation policy after convergence?
% 		Some related studies use models to compute energy function and project a specific action into the control safe set every time the agent interacts with the environment. Unlike them, 
		We claim that we can learn the energy function in a model-free manner similar to learning a value function. By using the energy function transition as the constraint objective, we formulate a constrained RL problem. We prove that our Lagrangian-based solutions make sure that the learned policy will converge to the constrained optimum under some assumptions. The proposed algorithm is evaluated on both the complex simulation environments and a hardware-in-loop (HIL) experiment with a real controller from the autonomous vehicle. Experimental results suggest that the converged policy in all environments achieve zero constraint violation and comparable performance with model-based baseline.
	\end{abstract}
	
% 	\begin{IEEEImpStatement}
% 		This paper propose an algorithm which is able to learn a zero-violation safe policy with limited prior information in  reinforcement learning. Our novelty is that we give a possible solution to the long-standing paradox when considering safety in the trail-and-error learning mechanism: how to learn a safe policy without necessarily experiencing it?
% 		This paradox forces that one must either allow a few violation with model-free RL, or require model knowledge for a solid safe policy. By learning a safety-oriented energy function in a model-free manner and enforcing the safe set constraint with a Lagrangian-based constrained optimization, our approach dramatically reduces the required prior knowledge to obtain a solid safe policy. It is not only advanced in theory, but also gives a promising approach to unleash the huge potential of RL on a variety of real-world safety-critical tasks, for example, autonomous driving demonstrated in our experimental results.
% 	\end{IEEEImpStatement}
	

	\section{Introduction}
	
	\IEEEPARstart{R}{einforcement} learning has drawn rapidly growing attention for its superhuman learning capabilities in many sequential decision making problems like Go \cite{silver2016mastering}, Atari Games\cite{mnih2013playing}, and Starcraft\cite{vinyals2019grandmaster}. However, such a powerful technique has not been quickly put into widespread real-world applications.
	%%has not been widely applied to real world applications.
	% , due to the lack of  safety assurance, even with the converged policy. 
	For instance, most of the remarkable improvements in robotics control \cite{haarnoja2018softb}, autonomous driving \cite{chen2021interpretable}, or surgical robotics \cite{richter2019open}, are only limited in the simulation platforms. One of the major issues of the real-world applications of RL algorithm is the lack of safety assurance, which can be summarized as the notorious paradox: how to learn a safe policy without enough data or prior model about the dangerous region?
	
% 	\begin{wrapfigure}{l}{0.5\linewidth}
%         \vspace{-20pt}
%         \begin{center}
%         \includegraphics[width=\linewidth]{fig/nonzero.png}
%         \end{center}
%         \vspace{-15pt}
%         \caption{Constrained RL with \emph{zero constraint threshold }. All baseline algorithms can not converge to a zero constraint-violation policies.}
%         \vspace{-10pt}
%     \end{wrapfigure}
    
	Most model-free RL with safety constraints studies interacts with the environment through the cost function, a similar environment response like the reward function, to formulate safety measurements \cite{altman1999constrained}. A commonly used cost function is the pulse signal, which returns 1 if the action is dangerous, otherwise 0 \cite{geibel2005risk, achiam2017constrained, ray2019benchmarking,yang2020projection,tessler2018reward}. This pulse signal is a \emph{posterior} measurement of safety, which means that the agent is penalized only after taking the dangerous action. It causes that one must collect \emph{many dangerous samples} to estimate the value function or advantage function of cost, and then provide danger-aware gradient to update policy. We provide a training sample in a commonly used safe RL benchmark, Safety Gym in Figure \ref{fig:intro}. the constraint threshold is set as zero, but the model-free constrained RL baselines cannot ensure zero violation \emph{even after convergence}. Similar problem also happens in other constraint formulation like risk-sensitive or chance constraint \cite{yang2021wcsac, chow2017risk}. Therefore, most model-free constrained RL studies allow \emph{a small tolerance of constraint violation} in their problem settings \cite{achiam2017constrained, chow2017risk, ray2019benchmarking}. This danger tolerance is not acceptable when the trained policies are applied to real-world safety-critical problems.% Two major reasons can explain that why these methods can nod handle the zero violation tasks: (1) the cost signal is too sparse to affect gradient and further confine the policy update; (2) the problem design is so conservative that a meaningful policy can hardly be obtained. Intuitively, the agent will not cause any collision if it just stands still; the combination of a constraint threshold at zero and the pulse cost function usually ends in a policy which is not willing to move anywhere \cite{ray2019benchmarking}.
		\begin{figure}[ht]
        \centering
        \includegraphics[width=0.6\linewidth]{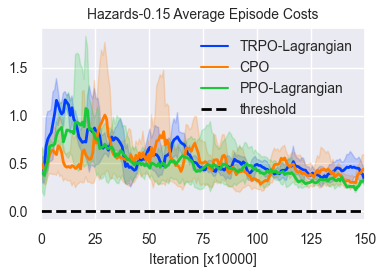}
        \caption{Model-free constrained RL baselines with \emph{zero constraint threshold}. All baseline algorithms can not converge to a zero constraint-violation policies.}
        \label{fig:intro}
    \end{figure}
    
	There are studies which adopt model information to predict future trajectory to enhance safety rather than taking the risk to explore the danger \cite{berkenkamp2017safe, duan2019deep, chen2017constrained}. The most widely used branch in safety enhancement of RL is to use the energy-function-based safety certificate to construct a safety monitor, or a safe controller. The energy function is assigned to be high for those safe states, and typical energy function definitions includes control barrier function (CBF) \cite{ames2019control,cheng2019end,ma2021model,qin2021learning} and \emph{safety index} in the safe set algorithm (SSA) \cite{liu2014control,zhao2021modelfree}. They first determine the \emph{control safe set} in the action space as the set of actions to make the state safe set attractive and forward invariant, which is achieved by making the energy function dissipate. If the base controllers (the policies in RL problems) output actions outside the control safe set, the safe controllers find the optimal safe actions through solving convex optimizations, for example, projecting the action into the control safe set \cite{cheng2019end}. After the projection, the persistent safety, or \emph{forward invariance}, is guaranteed \cite{wei2019safe}. Therefore, the energy function is a prior measurement of danger and enhances safety inside the safe region in the state space. Although providing provable safety guarantee, most energy-function-based safe controller relies on explicit dynamics model to execute the action projection, for example, learning-calibrated models \cite{cheng2019end}, specific type of models \cite{fisac2018general} or linearity assumption \cite{dalal2018safe}. Applying these methods to general complex tasks with no prior models is challenging. 
	
	% In this paper, we manage to find a solution to learn a zero constraint-violation policy  % Take the case of safe set algorithm, where the safe action set is defined by $\{a \mid \phi(s')<\max \{\phi(s)-\eta, 0\}\}$ with the safety certificate $\phi(\cdot)$, named as the safety index \cite{liu2014control}; All the necessary information for computing the safety index $\phi(\cdot)$ is \emph{available in the observations} of a RL agent, which means that we can learn the safety index transaction from sampled data. % In another word, to obtain the full model knowledge is actually a more difficult task than only to avoid the constraint violations, since these energy function-based algorithms only require the transitions of the safety certificate, or the energy function. 
	In this paper, we proposed the safe set actor critic (SSAC), a model-free RL algorithm aiming to learn a zero constraint violation policy without need of consistently experiencing the dangerous region or prior model. Without dynamics model to construct the safety monitor, we formulate a constrained RL problem with the constraints to dissipate the energy, called the \emph{safe action constraints}. We directly learn the energy function, or safety index in a model-free manner similar to the value function learning. In this way, we can simultaneously learn the safety index and the optimal safe policy. We solve the constrained RL with Lagrangian-based approach, and invoke our prior study to handle the state-dependent constraints with neural multipliers. We add a new proof of convergence to the optimal safe policy. The main contribution of this paper are listed as follows:
	\begin{enumerate}
		\item We propose the safe set actor critic (SSAC) algorithm, which aims to learn a zero constraint-violation policy without sufficient dangerous samples or prior model. We learn the safety-oriented energy function from sampled transitions and enforce the policy to dissipate the energy. The energy function is a prior measurement of safety that allows us to learn a safe policy without consistently violating the safety constraint in the complex environment with unknown dynamic model.
		\item Unlike other studies using the energy function as an additional controller, we formulate the RL problems as a constrained optimization with control safe set constraints. We use the Lagrangian-based approach with a multiplier neural network proposed in our prior research to solve this problem \cite{ma2021feasible}. Additionally, we provide a convergence proof that the learned policy converges to the optimal safe policy, while the safety monitor studies only guarantee local optimality on the action space.
		\item We evaluate the proposed algorithm on both the simulations on the state-of-the-art safe RL benchmarks, Safety Gym and autonomous driving tasks with a hardware-in-loop (HIL) system. Experiment results show that policies learned by the proposed algorithm achieve stable zero constraint violation in all the experimental environments.
	\end{enumerate}
	The paper is organized as follows. Section \ref{sec:2} introduces the related work. Section \ref{sec:3} gives our problem formulation. Section \ref{sec:4} is the practical algorithm and convergence analysis. Section \ref{sec:6} demonstrates the experimental results, and Section \ref{sec:7} concludes the paper. %is the details about how to learn the safety index transaction. Section \ref{sec:5}
	
	\section{Related Work}
	\label{sec:2}
	This section introduces the related work. Existing safe RL studies can be classified into two categories, constrained RL and RL with safe controllers. Section \ref{sec:related1} introduces the studies of constrained RL, and Section \ref{sec:related2} discusses about the RL with safe controllers.
	\subsection{Constrained Reinforcement Learning}
	\label{sec:related1}
	 Most of constrained RL studies take the form of constrained Markov decision process (CMDP), which enforces the constraint satisfaction on the \emph{expectation} of cost function on the trajectory distribution while maximizing the expected return or reward \cite{altman1999constrained,uchibe2007constrained,chow2017risk,achiam2017constrained,ray2019benchmarking,stooke2020responsive,tessler2018reward,ma2021model,yang2021wcsac,peng2021separated,zhang2020first,yang2020projection,duan2019deep,bhatnagar2009natural,ding2020natural,ding2021provably}. As they consider the safety-critical RL problem as a constrained optimization, and various constrained optimization techniques are applied in the constrained RL methods. For example, feasible descent direction \cite{uchibe2007constrained,achiam2017constrained,yang2020projection}, primal-dual optimization \cite{chow2017risk, ray2019benchmarking, stooke2020responsive,zhang2020first,ding2020natural,ding2021provably}, penalty function \cite{tessler2018reward,peng2021separated}. The expectation-based constrained objective function is not applicable for the state-dependent constraint commonly used in the safety-critical control problems, and our prior work propose a neural multiplier to handle the state-dependent safe constraints in RL \cite{ma2021feasible}. However, our prior study still use the value function $v_C^\pi(s)$ for state $s$ under policy $\pi$, or the discounted return of cost function as the constraint objective:
	 \begin{equation}
	     v_C^\pi(\state)=\E_{\tau\sim\pi}\Big\{\sum\nolimits_{t=0}^{\infty}\gamma^t c_t|s_0=s\Big\}
	 \end{equation}
	 The value-function-based constraints still fail to ensure zero violation for the reason stated above. By introducing the control safe set condition as a state-dependent safety constraint, we avoid the sparsity or conservativeness issues, and the proposed algorithm is able to learn a policy with zero constraint violation while accomplishing the task.
	\subsection{RL with Safe Controllers}
	\label{sec:related2}
	RL with safe controllers, or shielding RL means that the policy output is modified or projected to a safe action before execution \cite{cheng2019end, dalal2018safe, fisac2018general, pham2018optlayer,ferlez2020shieldnn}. In some studies, the control safe set is determined by the energy-based methods, for example, the control barrier function \cite{Agrawal2017a}, the safe set algorithm \cite{liu2014control}. These energy-function-based methods can provide provable safety guarantee by enforcing the \emph{forward invariance} of at least a subset of the state safe set \cite{ames2019control,zhao2021modelfree}. Nevertheless, it is difficult to find the safe projection, for example, projecting the output to a safe set on the action space \cite{cheng2019end}, or adding a safe layer to the policy network \cite{dalal2018safe}. However, to obtain the safe action and to prove forward invariance of the state safe set pose limitations on the model formulation, for example, it usually requires explicit knowledge \cite{cheng2019end}, kinematic bicycle model \cite{fisac2018general} or linear property \cite{dalal2018safe}. A sample-based safe set boundary estimation is proposed for black-box model \cite{zhao2021modelfree}, but it requires a digital-twin simulator of the environment. The novelties of our approach to handle energy function are (1) We learn the transition model of energy function similar like the Q-function learning in a model-free manner, so we no longer rely on explicit or implicit dynamics model; (2) We formulate constrained RL problems to learn constrained optimal \emph{policies} rather than locally optimal safe \emph{actions}. % Then the value function can be used as the constraint objective functions \red{as done in} our previous study  \cite{ma2021feasible}.
%	\subsection{Learning-based Safety Certificate}
%	Similar to the Lyapunov function in the stability analysis, it is difficult to synthesize the safety certificate even with a known dynamic system. Some recent studies focus on the learning-based safety certificate including learning the energy function, or the barrier certificate \cite{saveriano2019learning, srinivasan2020synthesis}. Saveriano uses supervised learning to learn a barrier function for the robotics motion planning, but it has a linearity assumption of the barrier function formulation \cite{saveriano2019learning}. The supported vector machine is also considered as the module to learn a barrier function, but it requires a large amount of data labeled as safe or unsafe by human \cite{srinivasan2020synthesis}.  
	
%	\subsection{Relation with Previous Work}
%	In our previous work, we proposed the feasible actor critic (FAC) algorithm, which is able to solve the problem with state-dependent constraints by introducing a neural network approximation of Lagrange multiplliers. We claim that the not all state are possibly feasible, 
	\section{Problem Formulation}
	\label{sec:3}
	In this section, we first describe the notion, then discuss the problem formulation of reinforcement learning with the safe set safety constraints.
	\subsection{Notion}
	We consider the constrained reinforcement learning (RL) problems wherein an agent interacts with its environment with safety constraints. This problem can be formulated a Markov decision progress (MDP), defined by the tuple $(\mathcal{S}, \mathcal{A}, \mathcal{R}, \mathcal{C}, \mathcal{F})$. The state space $\mathcal{S}$ and action space $\mathcal{A}$ are set to be continuous, and the reward function $r: \mathcal{S} \times \mathcal{A} \rightarrow \mathcal{R}$ maps a state-action pair to the rewards $r(s_t, a_t)$. Notably, the cost function $c: \mathcal{S} \times \mathcal{A} \rightarrow \mathcal{C}$ also maps a state-action pair to a cost function $c(s_t, a_t)$. It has a similar formulation with the reward function, which is used to formulate the constraint objective function later. For simplicity, and state-action pair in the next time step of $(\stp, a_{t+1})$ is denoted as $(s',a')$. $\mathcal{F}$ represents the unknown environment dynamics for state transition. We assume a deterministic environment dynamics (the deterministic assumption is common and reasonable in safe control problems). We consider that the agent starts from an initial state distribution $\initdist_{i}(s)$. The initial state set $\mathcal{I}=\{s|\initdist_i(s)>0\}\subseteq\mathcal{S}$ is the set of all possible state in the state distribution, where $f_i(s)$ is the probability density function for the initial state distribution. The parameterized function is represented by function with right subscripts, for example, the policy $\pi_\theta(\cdot)$ with parameter $\theta$.
	\subsection{Safe Set Algorithm}
	 The safe set algorithm (SSA) is a closely related design techniques for the safety-oriented energy function \cite{liu2014control} $\phi:\mathbb{R}^n\to\mathbb{R}$. The energy function, or the safety index are designed to satisfy (1) low energy states are safe and (2) there always exist actions to dissipate the energy. For a given energy function, the safe set algorithm indicates that the energy function can not increase rapidly, or exceed zero. The discrete-time constraints are:
	\begin{equation}
		\phi(s')<\max \{\phi(s)-\eta, 0\}
		\label{eq:ssacstr1}
	\end{equation}
	where the hyperparameter $\eta>0$ is a slack variable to confine how fast the energy function decreases. 
	\begin{definition}[Control Safe Set]
	The safe set on the action space is defined by 
		\begin{equation}
			\mathcal{U}_{S}^{D}(s):=\{a\in \mathcal{A} \mid \phi(s')<\max \{\phi(s)-\eta, 0\}\}
			\label{eq:safeset}
		\end{equation}
	\end{definition}
	We explain how to construct a safety index with the case of collision avoidance safety constraint, which is one of the most common safety constraints in real-world tasks. The safety measurement to define which states are safe is called the \emph{original safety index $\phi_0$. As for collision avoidance, the \emph{original safety index} is $\phi_0=d_{\text{min}}-d$, where $d$ denotes the distance between the agent and obstacle, $d_{\min}$ is the safe distance. 
	\begin{definition}[State Safe Set]
	The safe set on the state space includes state satisfying the original safety index:
	\begin{equation}
	    \mathcal{S}_{S}:=\{s\in \mathcal{S} \mid \phi_0(s)<0\}
	\end{equation}
	\end{definition}
	However, we may not guarantee all states in the state safe set to be safe persistently, which is also defined as the \emph{forward invariance} of the safe state set. One reason is that the input-to-constraint relation is high-order ($\frac{\partial\phi_0}{\partial u}=0$). In the collision avoidance case, if the input is some type of force like traction, the position constraint with force input is second-order. This high-order property would cause that there exist some states that would inevitably go unsafe, and we must exclude them and focus on a \emph{subset} of the state safe set. The commonly used form was proposed as adding high-order derivatives to the safety index.} The general parameterization rules for safety index is $\phi=\phi_0+k_1\phi_0^{'}+\cdots+k_n\phi_0^{(n)}$, where $\phi_0^{(n)}$ is the $n$-order derivative of $\phi_0$ \cite{liu2014control}. In this paper, we adopt a recent modified form of parameterization for the collision avoidance safety constraint in \cite{zhao2021modelfree}:
	\begin{equation}
		\phi(s)=\sigma+d_{\min }^{n}-d^{n}-k \dot{d}
		\label{eq:sis}
	\end{equation}
	where the high-order derivative term,  $\dot d$ is the derivative of distance with respect to time. % The time derivative can also be computed by the relative velocity between the agent and obstacle. 
	$\sigma,n,k>0$ is the parameters to tune. It is proved in \cite{zhao2021modelfree} that if the control safe set $\mathcal{U}_S^D(s)$ is always nonempty in $\mathcal{I}$, we can conclude that agent start from the initial set $\mathcal{I}$ would converge into a \emph{subset of state safe set $\mathcal{I}_s=\{s|\phi(s)\leq0\}\cap \{s|\phi_0(s)\leq0\}$, and  $\mathcal{I}_s$ is controlled forward invariant.} Therefore, we conduct the feasibility assumption:
		\begin{assumption}[Feasibility]
		The control safe set is nonempty for $\forall s \in \mathcal{I}$ where $\phi(s)$ is the safety index function of state $s$.
		\label{assumption:feasi}
	\end{assumption}
	In practical implementation, we follow the design rule of parameters $\sigma,n,k$ that the control safe set is nonempty all over the state space in \cite{zhao2021modelfree} to satisfy Assumption \ref{assumption:feasi}.
	
	\subsection{Reinforcement Learning with Control Safe Set Constraint}
	We consider a constrained reinforcement learning problem that maximizes the expected return while ensuring the control safe set constraint \eqref{eq:ssacstr1} on all its possible initial states:
	\begin{equation}
		\begin{aligned}
		    \max_{\pi}\ &J\big(\pi\big) =  \E_{\tau\sim\pi}\Big\{\sum\nolimits_{t=0}^{\infty}\gamma^t r_t\Big\} = -\E_\{\state\sim \initdist_i(\state)\}v^\pi(\state)
		\\
		\text{s.t.} \   &\phi(\state') - \max \{\phi(\state)-\eta, 0\}\leq 0, \forall s \in \mathcal{I}
		\end{aligned}
		\label{eq:statewiseop}
	\end{equation}
	where $v^\pi(\state)=\E_{\tau\sim\pi}\Big\{\sum\nolimits_{t=0}^{\infty}\gamma^t r_t|s_0=s\Big\}$ is the value function of state $\state$. 
	\begin{proposition}
		If $\pi^*$ is a feasible solution of RL problem \eqref{eq:statewiseop} with constraint \eqref{eq:cstrphi0}, then $\pi^*$ will not cause any constraint violation under Assumption \ref{assumption:feasi} if starting with a safe state $s\in\mathcal{I}_s$.
 	\end{proposition}
 	The proposition holds since $\pi^*$ is a control law that guarantees the forward invariance of the subset of state safe set $\mathcal{I}_s$. The state will never leave $\mathcal{I}_s$ to be unsafe.
	\section{Safe Set Actor Critic}
	\label{sec:4}
	In this section, we explain the solution methods of \eqref{eq:statewiseop} and practical algorithm. Section \ref{sec:learn1} introduce the Lagrangian-based approach with neural multipliers. Section \ref{sec:learn2} discusses about how we learn the transition model of safety index as a state-action value function. Section \ref{sec:algo1} demonstrates the practical algorithm and gradient computation, and Section \ref{sec:algo2} gives the convergence proof of the proposed safe set actor critic (SSAC) algorithm.
    \subsection{Lagrangian-based Solution}
    \label{sec:learn1}
    The major difficulty of solving problem \eqref{eq:statewiseop} is the number of constraints are infinite for continuous state space. Our prior work \cite{ma2021feasible} use the Lagrangian-based approach with \emph{neural multipliers} to approximate a mapping from states to multipliers to handle the problem (each state has a corresponding multiplier for the state-dependent control safe set constraints), denoted by the mapping $\lambda(s): \mathbb{R}^n\to\mathbb{R}$. The Lagrange function of \eqref{eq:statewiseop} is 
	\begin{equation}
		\begin{aligned}
		    L\big(\pi,\lambda\big) =& -\E_{\state\sim \initdist_i(\state)}v^\pi(\state) \\
		    &+ \sum\nolimits_{\state \in \mathcal{I}} \lambda(\state)\Big(\phi(\state') - \max \{\phi(\state)-\eta, 0\}\Big)
		\end{aligned}
		\label{eq:SL1}
	\end{equation}
	We take the negative of the reward part since standard Lagrangian function computes the minimum the optimization objective. The mapping $\lambda(s)$ can be approximated as neural network $\lambda(s; \xi)$ with parameters $\xi$. The sum over state in \eqref{eq:SL1} is unavailable since the state space is continuous. Our prior work \cite{ma2021feasible} solve this problem with an equivalent loss function
	\begin{equation}
	\begin{aligned}
	    	\bar{L}(\pi,\lambda) = \E_{\state\sim \initdist_i(\state)}\big\{&-v^{\pi}(\state) + \\
	    	&\lambda(\state)\big(\phi(\state') - \max \{\phi(\state)-\eta, 0\}\big)\big\}
	\end{aligned}
		\label{eq:SL2}
	\end{equation}
	We have proved the equivalence between two function: \eqref{eq:SL1} and \eqref{eq:SL2} in our prior work \cite{ma2021feasible}. Intuitively, the loss function \eqref{eq:SL2} can be regarded as a Lagrange function with reshaped constraints
	\begin{equation}
		\initdist_i(s)\left(\phi(\state') - \max \{\phi(\state)-\eta, 0\}\right)\leq 0
		\label{eq:reshapecstr}
	\end{equation}
	If the state lies in the initial state set, $\initdist_i(s)>0$ and the constraint \eqref{eq:reshapecstr} is equal to the original constraint. Otherwise, $\initdist_i(s)=0$, it means that those states are outside the initial state set, so their safety constraints are not considered. 
	\subsection{Learn the Safety Index Transition}
	\label{sec:learn2}
	The control safe set constraint \eqref{eq:ssacstr1} consists $\phi(s')$ which we need to use \emph{model} to predict. Related studies mostly use explicit model (both prior model or learning-calibrated model) to compute $\phi(s')$. We directly learn the transition of safety index, or the LHS of \eqref{eq:ssacstr1}. This learning is similar to learning a state-action value function with a zero discounted factor, or focusing on only the energy function transition on the instant step:
	\begin{equation}
		q^\pi_c(\state, a) \doteq \phi(\state') - \max \{\phi(\state)-\eta, 0\}
		\label{eq:cstrphi0}
	\end{equation}
    For consistency with the form of value function learning, we also define the RHS of \eqref{eq:cstrphi0} as the cost function:
    \begin{equation}
		c(\state, a) \doteq \phi(\state) - \max \{\phi(\state)-\eta, 0\}
		\label{eq:cost}
	\end{equation}
	It is easy to modify existing RL algorithm where $q^\pi_c(s, a)=\E_{\tau\sim\pi}\Big\{\sum\nolimits_{t=0}^{\infty}\gamma_c^t c_t|s_0=s,a_0=a\Big\}$ to learn safety index transition by setting the corresponding discounted factor $\gamma_c$ as zero. Similar to learn the Q-function in RL, we would use a neural network $Q_{w_c}(s, a)$ to approximate $q^\pi_c(s, a)$, which is discussed in the following section.
	
    We claim that the cost function \eqref{eq:cost} can be obtained in the transitions $(s, s')$ from samples in RL. Under the MDP formulation, these information should be contained in the observations of RL. Otherwise, the observation would not be enough to describe all the characteristics of the environment, which does not satisfy the requirement of a Markov decision progress. Take the case of collision avoidance of an autonomous vehicle, the distance relative speed to other traffic participants can be either directly measured by the Lidar or predicted by the perception module of the autonomous driving system. If some information is missing (for example, we do not know the speed of the surrounding vehicles), using current observation to make decision is inappropriate and dangerous, which is beyond our discussion. For example, according to the parameterization \eqref{eq:sis}, the required information for computing the safety index is $(d, \dot d, d', \dot d')$. 
%     After obtaining the necessary information from the observation, we can learn the constraint state-action value function, also the constrained objective in \eqref{eq:statewiseop} to be
% 	\begin{equation}
% 	\red{j^\pi_c(s_t, a_t)} \doteq \phi(\stp) - \max \{\phi(\st)-\eta, 0\}
% 		\label{eq:cstrphi}
% 	\end{equation}

% 	\label{sec:5}
	\subsection{Practical Algorithm}
	\label{sec:algo1}
	
	We develop practical implementation of SSAC based on the state-or-the-art off-policy maximum entropy RL framework. Compared to  unconstrained soft actor-critic \cite{haarnoja2018soft,haarnoja2018softb}, we learn two extra neural networks, a state-action value function representing the change of safety index $Q_{w_c}(s, a)$, and a statewise Lagrange multiplier $\Lambda_\xi(s)$. The detailed algorithm is demonstrated in Algorithm~1.
%	\begin{table}[h]
%		\centering
%		\begin{tabular}{cccc}
%			\hline
%			Function name & Function type & Notions & Weights \\
%			\hline
%			Value function & soft Q-function & $Q_\qparas(\st,\at)$ & $\qparas$ \\
%			Cost value function & Q-function & $Q_{\qparas_C}(\st,\at)$ & $\qparas_C$ \\
%			Policy & stochastic policy with $tanh$ bijector & $\pi_\theta(\st)$ & $\piparas$ \\
%			Multipliers & multiplier function & $\lambda_\xi(\st)$ & $\xi$ \\
%			\hline
%		\end{tabular}
%		\caption{Parameterized Function}
%		\label{tab:func}
%	\end{table}

		\begin{algorithm}[htb]
		\caption{Safe Set Actor-Critic}
		\label{alg:factotal}
		\begin{algorithmic}
			\Require $\qparas_1$, $\qparas_2$, $\qparas_C$, $\piparas$,  $\xi$.\Comment{Initial parameters}
			\State $\bar\diamondsuit$ $\leftarrow$ $\diamondsuit$ for $\diamondsuit\in\{$$\qparas_1$, $\qparas_2$, $\qparas_C$, $\piparas\}$ \Comment{Initialize target network weights}
			\State $\mathcal{B}\leftarrow\emptyset$ \Comment{Initialize an empty replay buffer}
			\For{each iteration}
			\For{each environment step}
			\State $\at \sim \pi_\piparas(\at|\st), \stp \sim \pdyn(\stp| \st, \at)$ 
			\Comment{Sample transitions}
			\State Compute $c(s_t, a_t)$ according to \eqref{eq:cost}
			\State $\mathcal{B} \leftarrow \mathcal{B} \cup \left\{(\st, \at, \reward(\st, \at), , c(\st, \at), \phi(s_t), \stp)\right\}$
			% \Comment{Store the transition in the replay buffer}
			\EndFor
			\For{each gradient step}
			\State $\qparas_i \leftarrow \qparas - \beta_Q \hat \nabla_{\qparas_i} J_Q(\qparas_i)$ for $i \ \in \{1,2\}$\Comment{Update the Q-function weights}
			\State $\qparas_c \leftarrow \qparas - \beta_Q \hat \nabla_{\qparas_i} J_{C}(\qparas_c)$\Comment{Update the safety index transition model}
			\If{gradient steps \texttt{mod} $m_\pi$ $=0$}
			\State $\piparas \leftarrow \piparas - \beta_\policy \hat \nabla_\piparas J_\policy(\piparas)$\Comment{Update policy weights}
			\State $\alpha \leftarrow \alpha - \beta_\alpha \hat \nabla_\alpha J_{\alpha}(\alpha)$ \Comment{Adjust temperature}
			\EndIf
			\If{gradient steps \texttt{mod} $m_\lambda$ $=0$}
			\State $\lamparas \leftarrow \lamparas + \beta_\lambda \hat \nabla_\lamparas J_\lambda(\lamparas)$\Comment{Update multipliers weights}
			\EndIf
			\If{parameters are outside the compact domain}
			\State Project the parameter into the domain
			\EndIf
			\State $\bar\diamondsuit$ $\leftarrow$ $\tau\diamondsuit+(1-\tau)\bar\diamondsuit$ for $\diamondsuit\in\{\qparas_1$, $\qparas_2$, $\qparas_C$, $\piparas\}$ \Comment{Update target network weights}
			\EndFor
			\EndFor
			\Ensure $\qparas_1$, $\qparas_2$, $\qparas_C$, $\piparas$,  $\xi$.
		\end{algorithmic}
	\end{algorithm}
	
% 	We need to train 4 types of parameterized functions: (1) two soft Q-function, $Q_{w_{*}}(s, a)$ where $*\in\{1,2\}$; (2) the value function of safety index transition $Q_{w_c}(s, a)$; (3) the stochastic policy $\pi_\theta(s)$; (4) the multipliers $\lambda_\xi(s)$. 
	In the training process, the loss and gradient computation of the soft Q-function $Q_{w_{*}}(s, a)$ are exactly the same as those in soft actor-critic\cite{haarnoja2018soft,haarnoja2018softb}. The loss of training the safety index transition is to learn a model to approximate how the value of safety index changes when the state transition occurs. The loss is formed similar with the state-action value function learning with a zero discounted factor (only focusing on the instant step):
	\begin{equation}
	\begin{small}
        \begin{aligned}
	     &J_{C}(w_c)\\ =&\mathbb{E}_{\left(\st, \at\right) \sim \mathcal{B}}\bigg\{\frac{1}{2}\Big(Q_{w_c}(s_t, a_t) - c(s_t,a_t)\Big)^{2}\bigg\}\\
	     =&\mathbb{E}_{\left(\st, \at\right) \sim \mathcal{B}}\bigg\{\frac{1}{2}\Big(Q_{w_c}(s_t, a_t) - (\phi(\stp) - \max \{\phi(\st)-\eta, 0\})\Big)^{2}\bigg\}
        \end{aligned}
        \end{small}
	\end{equation}
	
% 	 and the gradient with respect to the parameters
% 	 \begin{equation}
% 	 	\begin{aligned}
% 	 		J_{C}(w_c)= &\mathbb{E}_{\left(\st, \at\right) \sim \mathcal{B}}\bigg\{\frac{1}{2}\Big(j^\pi_{w_c}(s_t, a_t) -\\
% 	 		& c(s_t,a_t)\Big)^{2}\bigg\}
% 	 	\end{aligned}
% 	 \end{equation}
 	The stochastic gradient of parameter $w_c$ is 
 	\begin{equation}
		\hat{\nabla}J_{C}(w_c)=\nabla_{w_c}Q_{w_c}(s_t, a_t)(Q_{w_c}(s_t, a_t) 
	    	- c(s_t,a_t))
 	\end{equation}
 	The objective function of updating the policy and multipliers are the Lagrange function in \eqref{eq:SL2}. In practical implementations, we actually use $Q_{w_c}(s, a)$ to approximate the LHS of inequality constraints in  \eqref{eq:statewiseop}. Using the framework of maximum entropy RL, the objective function of policy update is:
 	\begin{equation}
 		\begin{aligned}
 			J_{\pi}(\piparas)=&\mathbb{E}_{\st \sim \mathcal{B}}\bigg\{\mathbb{E}_{\at \sim \pi_{\piparas}}\Big\{\alpha \log \big(\pi_{\piparas}\left(\at \mid \st\right)\big)-\\
 			&Q_{\qparas}(\st, \at) + \Lambda_\xi(\st) Q_{w_c}(s_t, a_t) \Big\}\bigg\}
 			\label{eq: losspi}
 		\end{aligned}
 	\end{equation}
 	where $\mathcal{B}$ is the state distribution in the replay buffer. The policy gradient with the reparameterized policy $\at=f_\piparas(\epsilon_t; \st)$ can be approximated by:
 	\begin{small}
 		\begin{equation*}
 			\begin{aligned}
 				\hat{\nabla}_{\piparas} J_{\pi}(\piparas)=&\nabla_{\piparas} \alpha \log \big(\pi_{\piparas}\left(\at \mid \st\right)\big)+\Big(\nabla_{\at} \alpha \log \left(\pi_{\piparas}\left(\at \mid \st\right)\right)\\
 				&-\nabla_{\at} \big(Q_\qparas\left(\st, \at\right) - \Lambda_\xi(\st) Q_{w_c}(s_t, a_t) \big)\Big) \nabla_{\piparas} f_{\piparas}\left(\epsilon_{t} ; \st\right)
 			\end{aligned}
 		\end{equation*}
 	\end{small} 
 	Neglecting those irrelevant parts, the objective function of updating the multiplier network parameters $\xi$ is 
 	\begin{equation*}
 		J_{\lambda}(\xi) = \mathbb{E}_{\st \sim \mathcal{B}}\bigg\{\mathbb{E}_{\at \sim \pi_{\piparas}}\Big\{\Lambda_\xi(\st)\big( Q_{\qparas_c}(\st,\at)\big)\Big\}\bigg\}
 	\end{equation*}
 	The stochastic gradient of the multiplier network parameters $\xi$ is
 	\begin{equation}
 		\hat{\nabla}J_\lambda(\lamparas) = Q_{w_c}(s_t, a_t)  \nabla_\lamparas\Lambda_\lamparas(\st)
 		\label{eq:lamsubgrad}
 	\end{equation}
 	
 	We invoke a multiple delayed update mechanism which update policy and multiplier each $m_\pi$ or $m_\lambda$ step, respectively. This delayed mechanism helps to stabilize the training process explained on the level of tricks \cite{fujimoto2018addressing}. Furthermore, in the following theoretical convergence analysis, we assume that the learning rate of multiplier net is much smaller than the policy learning rate. This delayed update mechanism can also be regarded as a practical implementation of Assumption \ref{assumption:1}.
 	\subsection{Convergence Outline}
 	\label{sec:algo2}
    The convergence of maximum entropy RL relies on the soft policy iteration \cite{haarnoja2018soft}. As for SSAC, the soft policy evaluation part is the same as the unconstrained soft actor-critic, which means that the Q-function $Q_w(s, a)$ and the safety index transition $Q_{w_c}(s, a)$ will converge to their real state-action value function $q^\pi(s,a), q^\pi_c(s,a)$.
%  	\begin{lemma}[Soft policy evaluation]
%  		The parameters of Q-function $Q_w(s, a)$ and the safety index transition $Q_{w_c}(s, a)$ converges almost surely to the optimal parameters $w^*, w_c^*$.
%  	\end{lemma} 
 	Therefore, we focus on the soft policy improvement. First, we give some necessary assumptions:
 	\label{sec:51}
 	\begin{assumption}
 	\label{assumption:1}
 		The learning rate schedules, $\{\beta_\theta(k), \beta_\xi(k)\}$, satisfy
 		\begin{equation}
 			\begin{gathered}
 				\sum_{k} \beta_\theta(k)=\sum_{k} \beta_\xi(k)=\infty \\
 				\sum_{k} \beta_\theta(k)^{2}, \quad \sum_{k} \beta_\xi(k)^{2}<\infty \\
 				\beta_\theta(k)=o\left(\beta_\xi(k)\right) .
 			\end{gathered}
 		\end{equation}
 	\end{assumption}
 	\begin{assumption}{(Differentiablity)}
 		$\theta \in \Theta,\ \xi\in\Xi,\ \zeta\in\mathcal{Z}$ which are all compact sets. All the neural networks approximation are Lipschitz continuous.
 	\end{assumption}
 	We use an ordinary differential equation (ODE) viewpoint to prove the convergence, which is a branch of standard convergence analysis of stochastic optimization algorithms \cite{borkar2009stochastic, bhatnagar2009natural, bhatnagar2012online, chow2017risk}.  Take the case of $\theta$ update rule in Algorithm~1:
 	\begin{equation}
 		\piparas \leftarrow \piparas - \beta_\theta \hat \nabla_\piparas J_\policy(\piparas)
 	\end{equation}
 	It can be regarded as a stochastic approximation of a continuous dynamic system $\theta(t)$:
 	\begin{equation}
 		\dot \theta = -\nabla_\theta J_\policy(\theta)
 		\label{eq:ode}
 	\end{equation}
 	and we will prove that dynamic system $\theta(t)$ will converge to a local optimum $\theta^*$, which is the the solution of the ODE \eqref{eq:ode} and also a stationary point.
 	\begin{theorem}
 		Under all the aforementioned assumptions, the sequence of policy and multiplier parameters tuple $(\theta_k, \xi_k)$ converge almost surely to a locally optimal policy and multiplier parameters tuple $(\theta^*, \xi^*)$ as $k$ goes to infinity.
 		\label{theorem:major}
 	\end{theorem}
 	The theoretical computational graph is shown in Figure \ref{fig:flow}.
 	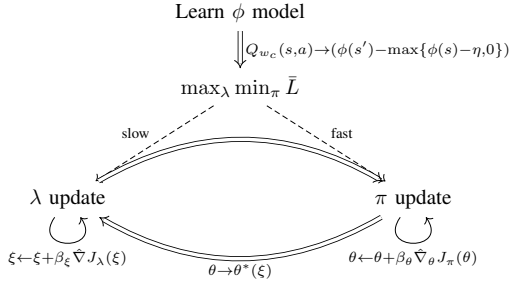
\begin{figure}[ht]
	   % \centering
	    \adjustbox{scale=0.8,center}{
	     \begin{tikzcd}                                                                                                    & \text{\rm Learn}\ \phi\ \text{\rm model} \arrow[d, "{Q_{w_c}(s,a)\to(\phi(s')-\max\{\phi(s)-\eta,0\})}", Rightarrow]       & 
	     \\& \max _{\lambda} \min _{\pi} \bar{L} \arrow[rdd, "\text{\rm fast}", dashed] \arrow[ldd, "\text{\rm slow}"', dashed] &                                                                                                                                                                                                                \\
	     &            &                    \\
	     \lambda\ \text{\rm update} \arrow["\xi\leftarrow \xi+\beta_\xi \hat{\nabla}J_\lambda(\xi)"', loop, distance=2em, in=305, out=235] \arrow[rr, Rightarrow, bend left] &                                                                                                                & \pi\ \text{\rm update} \arrow["\theta\leftarrow \theta+\beta_\theta\hat{\nabla}_{\theta} J_{\pi}(\theta)"', loop, distance=2em, in=305, out=235] \arrow[ll, "\theta \to \theta^*(\xi)", Rightarrow, bend left]
        \end{tikzcd}
	    }
	    \caption{The computational graph for SSAC.}
	    \label{fig:flow}
	\end{figure}
 	Before going detailed into the proof, the high level overview is gives as follows:
 	\begin{enumerate}
 		\item First we show that each update cycle of the multiple timescale discrete stochastic approximation algorithm $(\theta_k, \xi_k)$ converges almost surely, but at different speeds, to the stationary point $(\theta^*, \xi^*)$ of the corresponding continuous time system.
 		\item By Lyapunov analysis, we show that the continuous time system is locally asymptotically stable at $(\theta^*, \xi^*)$.
 		\item We prove that $(\theta^*, \xi^*)$ is locally optimal solution, or a local saddle point for the constrained RL problems.
 	\end{enumerate}
 	The convergence proof mainly follows the multi-timescale convergence and Theorem 2 in Chapter 6 of Borkar's book\cite{borkar2009stochastic}, and also the standard procedure in proving convergence of stochastic programming algorithms \cite{bhatnagar2009natural,bhatnagar2012online,prashanth2013actor}. Detailed proof of Theorem \ref{theorem:major} is provided in Appendix \ref{appendix:proof1}.
 	\section{Experimental Results}
 	\label{sec:6}
 	We demonstrate the effectiveness of proposed algorithm by evaluating them on multiple safety-critical tasks in this section. Section \ref{sec:exp1} demonstrates four safe exploration tasks on the state-of-the-art safe RL benchmark, Safety Gym \cite{ray2019benchmarking}, and Section \ref{sec:exp2} depicts a hardware-in-loop (HIL) experiment of autonomous driving tasks with a real controller module of autonomous vehicle. 
 	\subsection{Safe Exploration Task}
 	\label{sec:exp1}
 	\begin{figure}[ht]
 		\subfigure[Safety Gym]{\includegraphics[width=0.24\linewidth]{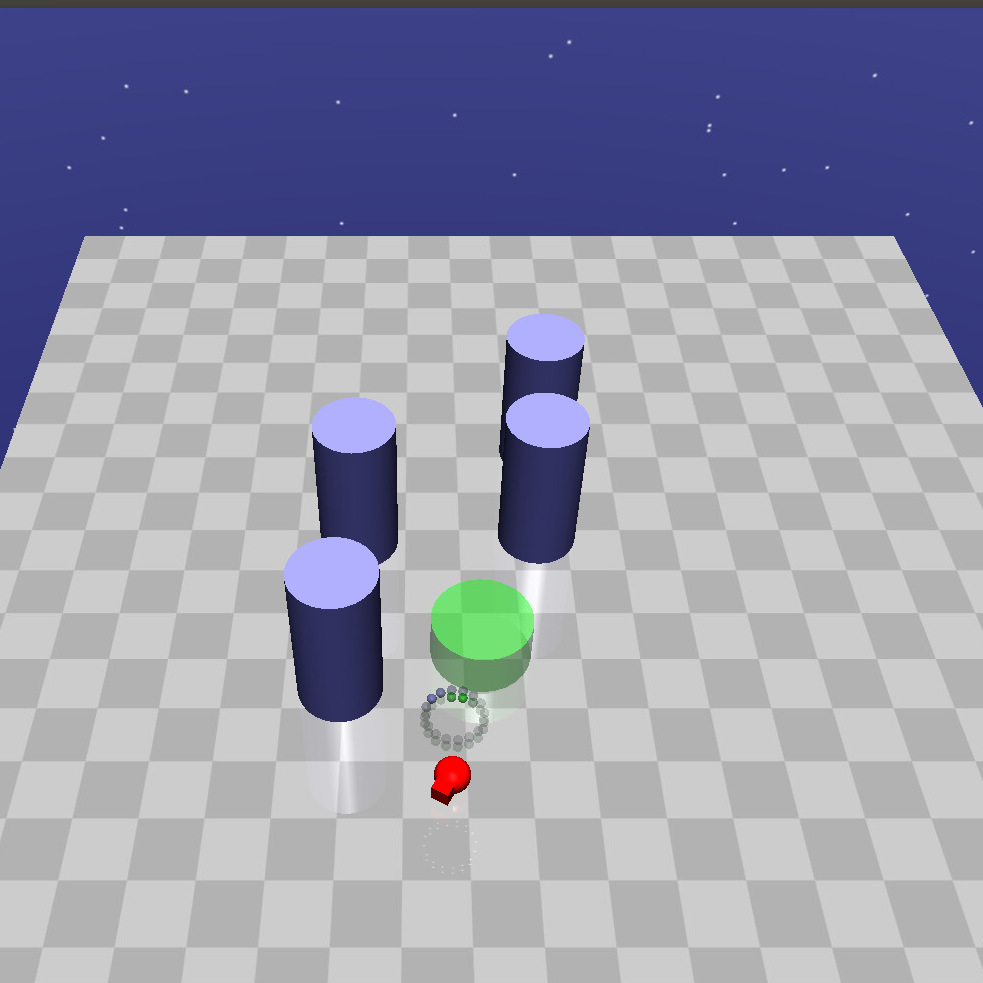}}
 		\subfigure[Point agent]{\includegraphics[width=0.24\linewidth]{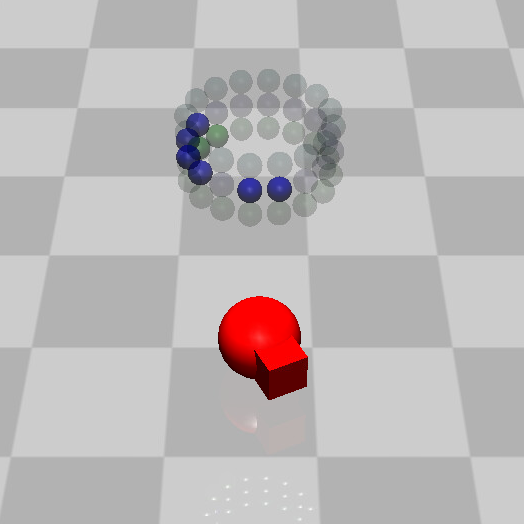}}
 		\subfigure[Hazard]{\includegraphics[width=0.24\linewidth]{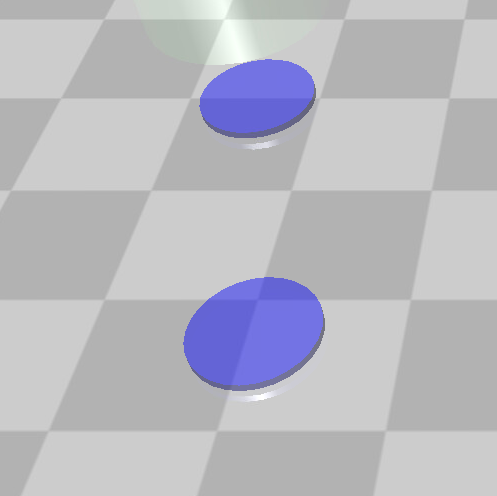}}
 		\subfigure[Pillars]{\includegraphics[width=0.24\linewidth]{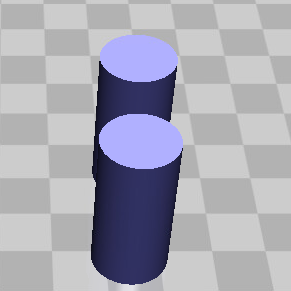}}
 		\caption{Safe exploration environments in Safety Gym. The agent should reach the goal position, the green cylinder in (a), while avoiding other obstacles. % Notably, hazards should be avoided but it is not a solid object.
 		hazards are not solid objects, hence can be crossed over. Pillars are solid and cause physical contact with agents.}
 	\end{figure}
 	\begin{figure*}[h]
 		\centering
 		\includegraphics[width=0.23\linewidth]{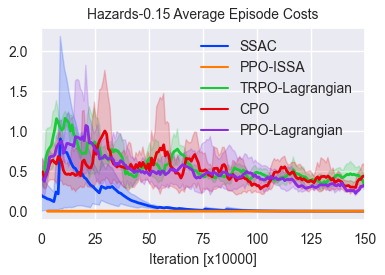}
 		\includegraphics[width=0.23\linewidth]{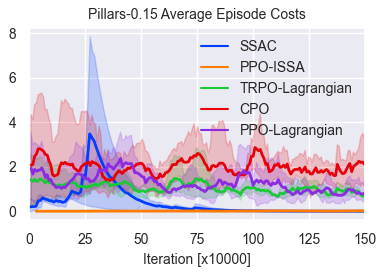}
 		\includegraphics[width=0.23\linewidth]{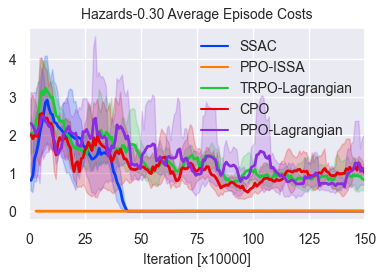}
 		\includegraphics[width=0.23\linewidth]{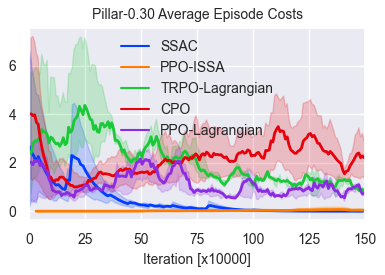}
 		\\
 		\includegraphics[width=0.23\linewidth]{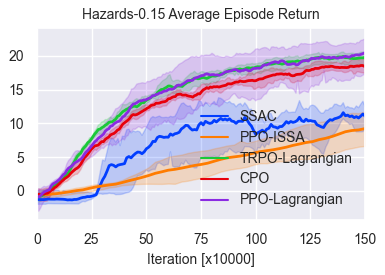}
 		\includegraphics[width=0.23\linewidth]{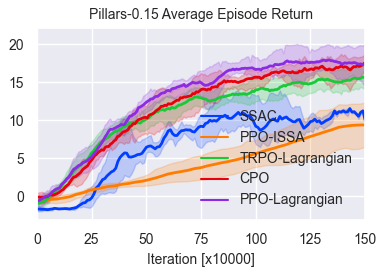}
 		\includegraphics[width=0.23\linewidth]{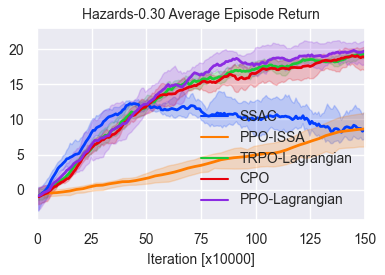}
 		\includegraphics[width=0.23\linewidth]{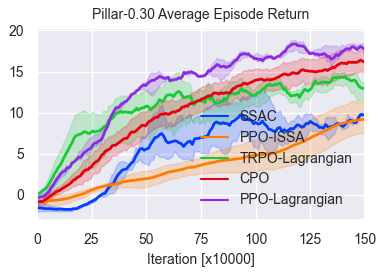}
 		\\
 		\includegraphics[width=0.23\linewidth]{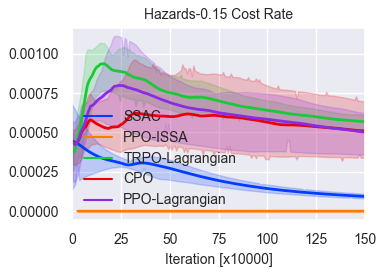}
 		\includegraphics[width=0.23\linewidth]{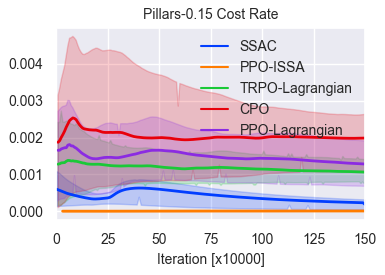}
 		\includegraphics[width=0.23\linewidth]{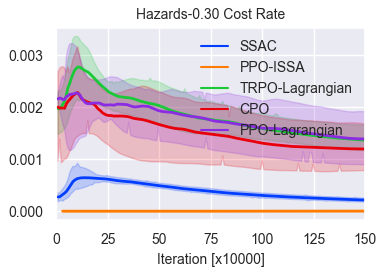}
 		\includegraphics[width=0.23\linewidth]{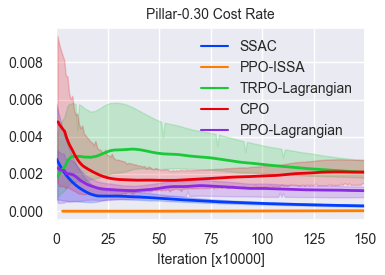}
 		\caption{Learning curves for safe exploration in Safety Gym environments. The x-axis represents the training iterations; and the y-axis represents the respective value in the title of the last 20 episodes. The solid lines represent the mean value over 5 random seeds. The shaded regions represent the 95\% confidence interval. SSAC converges to a \emph{zero constraint-violation} policy with a reasonable performance sacrifice while baseline algorithms all can not achieve zero violation.}
 		\label{fig:bigfig}
 	\end{figure*}
 	\subsubsection{Environment Setup}
 	Safety Gym includes a series of benchmark environments specially designed for safe RL and safe exploration tasks using Gym API and MuJoCo simulator \cite{brockman2016openai,todorov2012mujoco,ray2019benchmarking}. The task of safety gym is to control an agent (We choose the \texttt{Point} agent) to accomplish different types of tasks in a 2D arena. There exists hazards and different kinds of obstacles in the 2D arena with random positions, and the agent must avoid them when reaching the goal. Four environment setups with different types and sizes of the obstacles are considered, named as \texttt{\{Type\}-\{Size\}}. The safety constraint is not to touch any obstacles or reach any hazards while accomplishing the tasks. The parameters we choose for \eqref{eq:sis} satisfies the safety index synthesis rule proposed in \cite{zhao2021modelfree}, which satisfies the feasibility in  Assumption \ref{assumption:feasi}. % ,  that is all the initial states are feasible since the control safe set are always nonempty
 	
 	 We compare the proposed algorithm against different types of baseline algorithms: (1) commonly used constrained RL baseline algorithms: constrained policy optimization (CPO)\cite{achiam2017constrained}, Lagrangian modification version of trust-region policy optimization (TRPO-Lagrangian) and proximal policy optimization (PPO-Lagrangian) \cite{ray2019benchmarking}; (2) A model-free safe monitor which projects the action through an adaptive boundary estimation of control safe set, called the implicit safe set algorithm (PPO-ISSA) \cite{zhao2021modelfree}. As we do not have the dynamics model for Safety Gym, other safe controller baselines can not be implemened. Other experimental details are listed in Appendix \ref{sec:appendixexp}.
     
 	\subsubsection{Simulation Results}
 	The training curves in Safety Gym environments can be found in Figure \ref{fig:bigfig}. Results show that the converged policy of SSAC achieves solid zero constraint violation, while all constrained RL baseline in type (1) show little but existing constraint violations \emph{with a zero constraint threshold}. Furthermore, although we can not guarantee the zero constraint violation during the training progress, policies trained by SSAC only experience a few constraint violations in the initial stage of training, which means our proposed algorithm has the potential to achieve full constraint-avoidance training process given a proper initialization rule of feasible policy. The implicit safe set algorithm (PPO-ISSA) has comparable reward performance while learning slower than SSAC, which shows that learning a integrated policy (optimum on the policy space) and using a external safe monitor (optimum on local action space) nearly has same reward performance. PPO-ISSA has slower learning process, which might be explained by that the initial violations feed SSAC with samples of better reward performance. The total reward of SSAC is not as good as baselines, which is reasonable since some baseline achieves higher reward because they direct cross the hazards instead of avoiding them \cite{achiam2017constrained, ray2019benchmarking}. % many existing studies shows that there exists trade-off between the reward and safety performance  .  In real-world safety-critical tasks, a conservative but solidly safe policy is indeed acceptable.
 	
 	In order to better understand the zero violation property, we count the accumulative cost rate in the training process in Figure \ref{fig:bigfig}. Results reveal that the cost rate of SSAC \emph{converges to zero}, while the curves of baseline algorithms remain at a positive number. The facts verify that baseline algorithms need consistent cost signal to identify danger, while SSAC does not need explorations in dangerous region once converged.
 	\subsection{Autonomous Driving Control Task}
 	\label{sec:exp2}
 	\subsubsection{Environment Setup} The environment part follows our prior work \cite{ma2021model, guan2021integrated} modeling a real intersection located at ($31^{\circ}08'13''$N, $120^{\circ}35'06''$E) in the SUMO network editor  as shown in Figure \ref{fig: scene model}.
 	\begin{figure}
 		\centering
 		\subfigure[Real scene]{\includegraphics[width=0.45\linewidth]{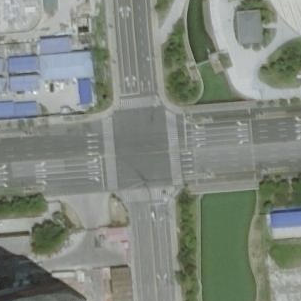}}
 		\subfigure[Virtual scene]{\includegraphics[width=0.45\linewidth]{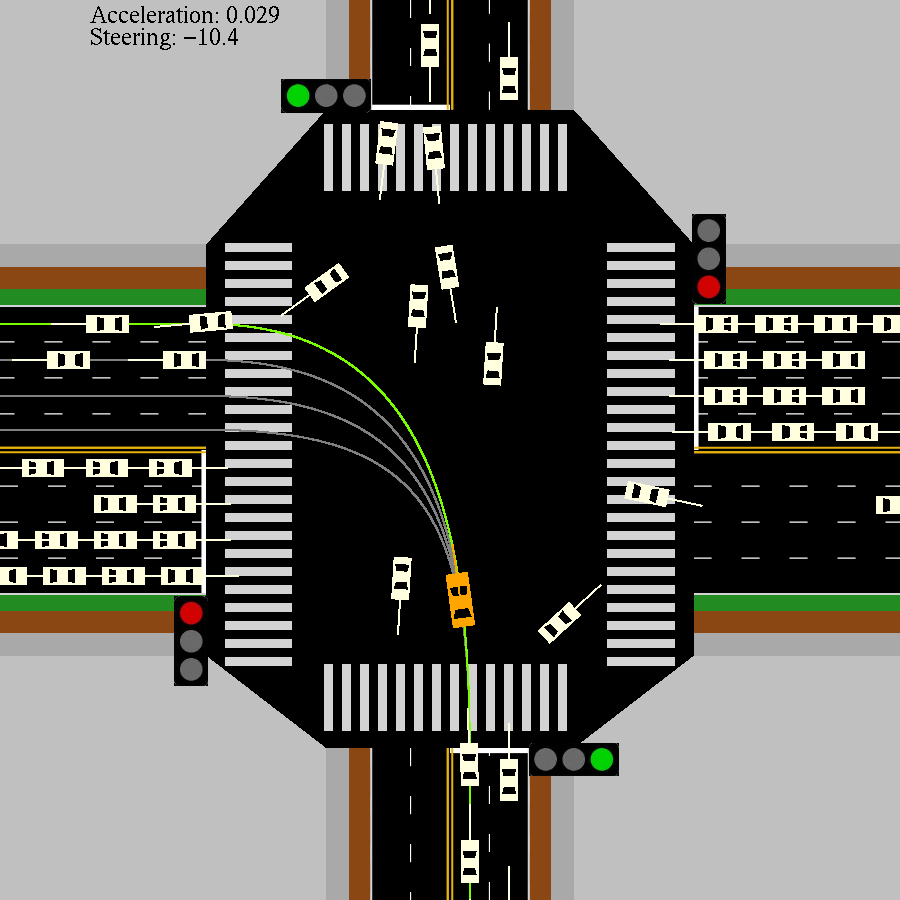}}
 		\label{fig: scene model}
 		\caption{The environment scene of autonomous driving control of turning left at intersection. The ego vehicle (yellow one) needs to track one of the reference paths, or the highlighted one while avoiding collisions with the surrounding vehicles (those white ones).}
 	\end{figure}
 	The random traffic flow is generated by the SUMO traffic simulator. The vector-level observation (not raw data from Lidars) includes three parts, the state of ego vehicle, a series of tracking point from a pre-defined static reference trajectory, and also the state of surrounding vehicles (including positions, velocity and heading angle). Detailed of observations are shown in TABLE \ref{tab:obs}.
 	\begin{table}[h]
 	\centering
 	\caption{State in the intersection environment}
 	\label{tab:obs}
    \begin{tabular}{@{}ccc@{}}
    \toprule
    Observations       & Notions & Meanings                        \\ \midrule
    Ego vehicle        &    $p_x$     & Longitudinal position {[}m{]}   \\
                       &    $p_y$     & Lateral position {[}m{]}        \\
                       &    $\eta$     & Heading angle {[}$\degree${]}         \\
                       &    $v_x$     & Longitudinal velocity {[}m/s{]} \\
                       &    $v_y$     & Lateral velocity {[}m/s{]}      \\
                       &    $r$     & yaw rate {[}rad/s{]}            \\ \midrule
    Surrounding vehicle&    $p_x^i$     & Longitudinal position {[}m{]}   \\
        ($i^\text{th}$)&    $p_y^i$     & Lateral position {[}m{]}        \\
                       &    $v^i$     & velocity {[}m/s{]}              \\
                       &    $\eta^i$     & Heading angle {[}$\degree${]}         \\ \midrule
    Tracking error     &    $\delta_p$     & position error {[}m{]}          \\
                       &    $\delta_\eta$     & Heading angle error {[}$\degree${]}   \\
                       &    $\delta_v$     & velocity error {[}m/s{]}        \\ \bottomrule
    \end{tabular}
    \end{table}
    The control input is desired steering wheel angle $\delta_{\text{des}}$ and acceleration $a_{\text{des}}$ of ego vehicle. The static reference trajectories are generated by bezier curves with several reference points according to the road topology. We selected \emph{focusing surrounding vehicles} in all the surrounding vehicles to handle the variable surrounding vehicles and safety constraints. The surrounding vehicle is filtered and listed in a specific order (first by route, then by lateral position) to avoid the permutation error in state representation. We name them as the \emph{filtered surrounding vehicles} in the sections below. The filtering rule is a fixed number of surrounding vehicles in each route, for example, 4 vehicles turning left from current road. If there are not enough vehicles, we add virtual vehicles in the filtered surrounding vehicles in distant place (outside the intersection) to make up the state dimension without affecting the ego vehicle \cite{guan2021integrated}. Notably, the position, heading angle and speed of both ego vehicle and filtered surrounding vehicles are included in the observations, and we can use them to compute the necessary distance term and its first-order time derivative in \eqref{eq:sis}. We use a kinetics bicycle vehicle model with nonlinear tire model to simulate the ego vehicle response \cite{ge2020numerically}, and the filtered surrounding vehicles are modeled and controlled by SUMO.  The reward is designed with linear combination of squared tracking errors of position and heading angle with respect to the highlighted reference trajectory in Figure \ref{fig: scene model}:
    \begin{equation}
    r(s, a) = -(0.04\delta_p^2+0.01\delta_v^2+0.1\delta_\phi^2 + 0.1\delta^2 + 0.005a_{\text{des}}^2)
\end{equation}
    The safety constraints are considered with a double circle design as shown in Figure \ref{fig:cstr}. For each pair of ego and filtered surrounding vehicle, the 
 	\begin{figure}[h]
 		\centering
 		\includegraphics[width=0.7\linewidth]{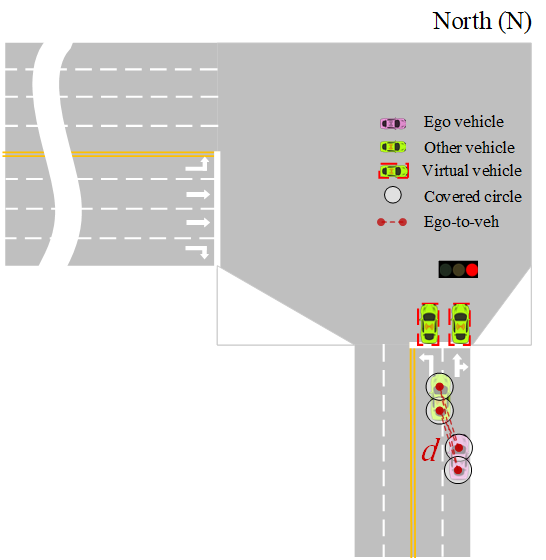}
 		\caption{Collision avoidance constraints by double circle.}
 		\label{fig:cstr}
 	\end{figure}
 	As we have multiple collision avoidance constraints here (4 for each filtered surrounding vehicle), we set the output dimension of multiplier network to match the constraint dimension (4 * number of filtered surrounding vehicles). Other environment setup details and specific weights design can be found in our previous studies \cite{guan2021integrated, ma2021model}.
 	
 	\subsubsection{Hardware-in-Loop Simulation System}
 	We design a hardware-in-loop system to verify the proposed algorithm. The controller is GEAC91S series AI computing unit from TZTEK as shown in Figure \ref{fig:hil}, based on a NVIDIA Jetson AGX Xavier. First, the policy is trained on a personal computer or cloud computing services, then we convert the neural network to a hardware-compatible model and implement it on the real autonomous driving controller.
 	\begin{figure}[h]
 		\centering
 		\includegraphics[width=0.8\linewidth]{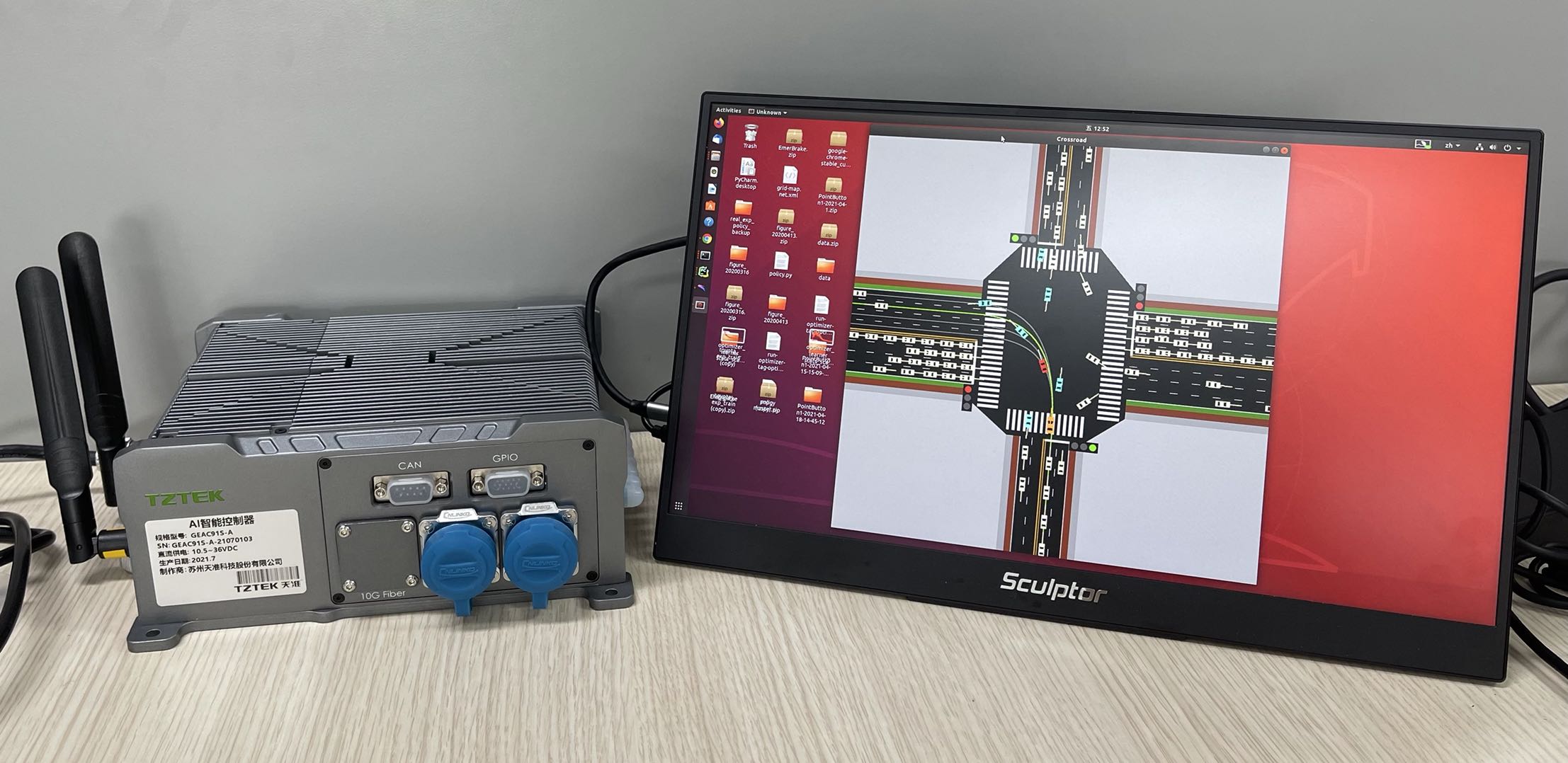}
 		\caption{HIL system with real autonomous vehicle controller.}
 		\label{fig:hil}
 	\end{figure}
 	 The neural network model of the policy is converted from pb model (trained and saved in TensorFlow) to TensorRT engine, which is optimized and accelerated by TensorRT, an inference framework from Nvidia. Notably, we do not train the model on the controller. We use a communication message queue, ZeroMQ to connect the controller module and a personal computer with simulation environments shown in Figure \ref{fig:hil}. At each step in the HIL experiment, the simulated observations from simulation environments are transmitted to controller and taken as inputs by the network. The network then output the control command using the inference API inside the controller. The control action is sent back to the simulation environment using the communication message queue.
 	\subsubsection{Training Results}
 	The training process is purely on the simulation environments before implementing on the hardware, and the learning curves are shown in Figure \ref{fig:exp2curve}. The implicit safe set algorithm requires a simulator that could be reset to a previous state \cite{zhao2021modelfree}, and SUMO simulator does not support the vehicles reset. Therefore, we do not use the safe controller baseline, PPO-ISSA for this experiment. The Episodic costs performance is similar to those safe exploration tasks that SSAC consistently enforces the episodic costs staying near zero. Figure \ref{fig:exp2curve} indicates that there are still about 0.5\% cost rate. After carefully monitoring the simulation process, the reason of these occasional costs is that surrounding vehicles would occasionally collide into ego vehicle even the ego vehicle has stopped. It does not happen frequently or every time the ego vehicle stops, so it is reasonable that SSAC does not learn to avoid it.\footnote{ We have tried to adjust the conservativeness parameter in SUMO but it doesn't work.} The overall return shows that SSAC does not have a distinguishable performance sacrifice, all the algorithms finish the tracking task. We further count the cost rate and cumulative cost to further understand the zero constraint violation property. The cost rate of proposed algorithm \emph{converges to zero}, while the curves of other algorithms remain at a positive number. Similarly, the cumulative costs are consistently increasing during the training procedure, while SSAC merely receives new costs after the initial training stage.
 	\begin{figure}[h]
 		\centering
 		\includegraphics[width=0.49\linewidth]{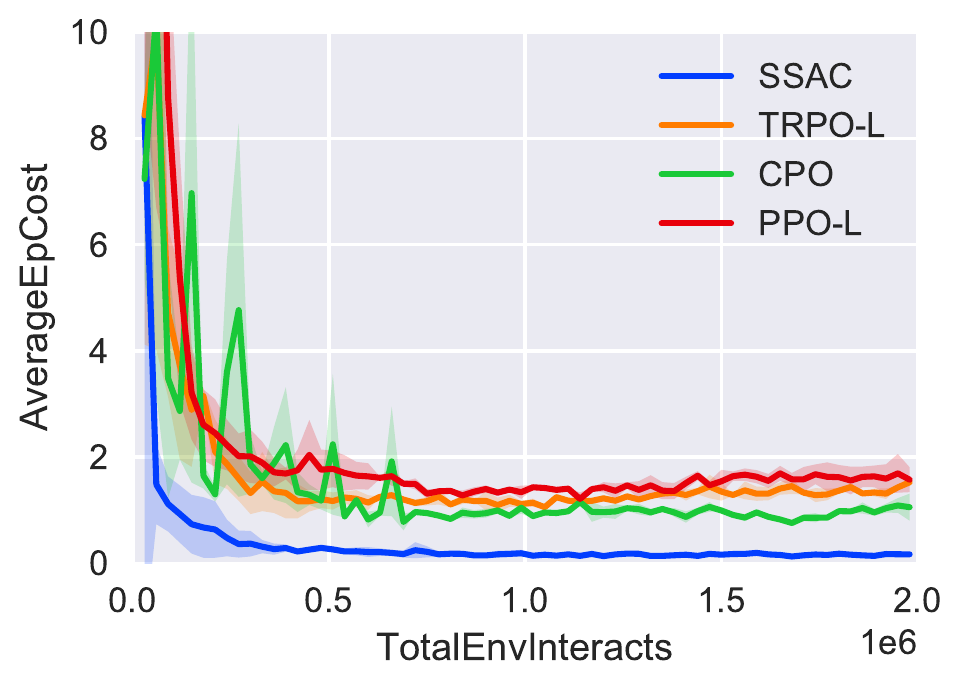}
 		\includegraphics[width=0.49\linewidth]{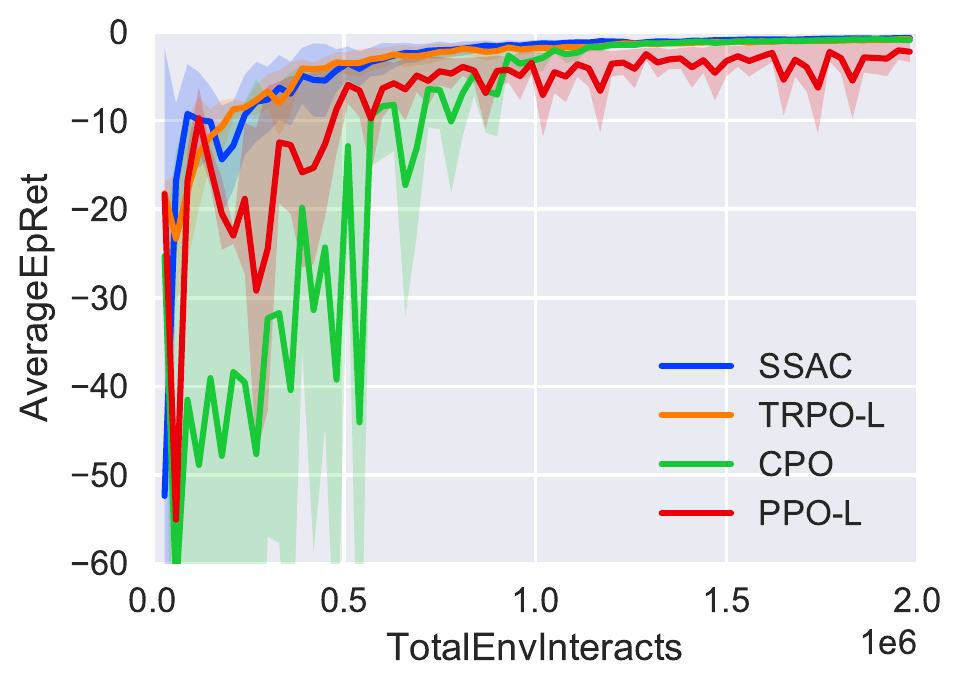}
 		\includegraphics[width=0.49\linewidth]{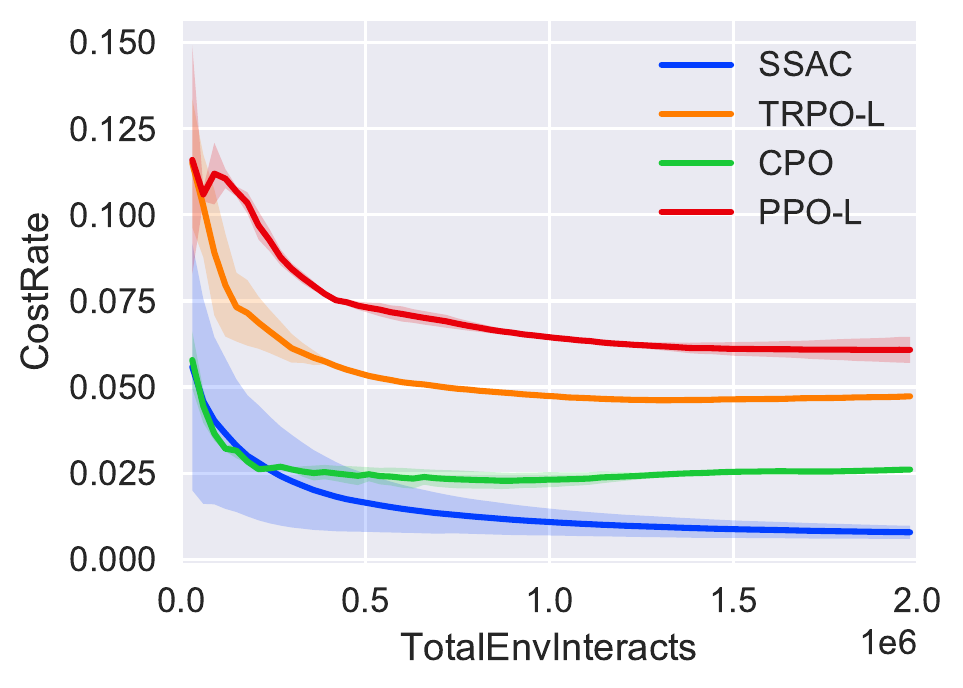}
 		\includegraphics[width=0.49\linewidth]{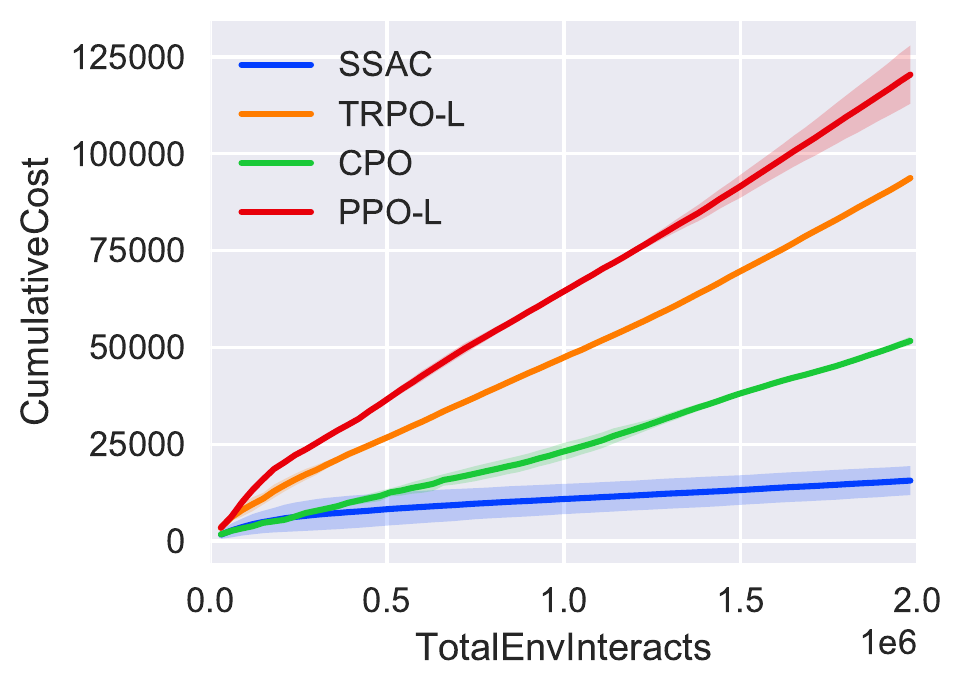}
 		\caption{Training curves of the offline training stage. All algorithm achieves near-zero tracking performance, but all baseline algorithms still have a small but non-negligible constraint violation. The occasional constraint violations of SSAC is caused by flaws in SUMO.}
 		\label{fig:exp2curve}
 	\end{figure}

 	\begin{figure}[h]
 		\centering
 		\includegraphics[width=\linewidth]{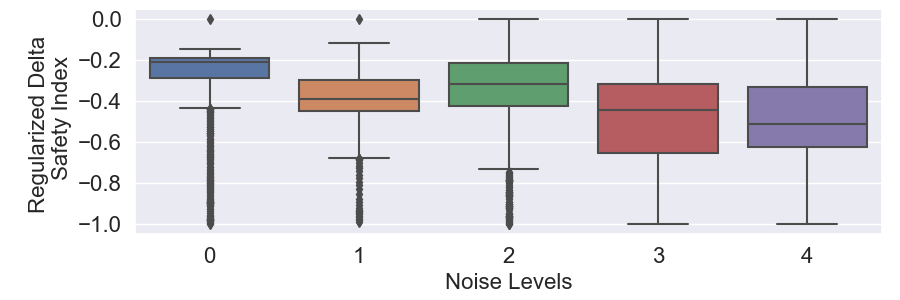}
 		\caption{Safety index transition value $\phi(s') - \max \{\phi(s)-\eta, 0\}$ distribution with different observation noise levels.}
 		\label{fig:boxphi}
 	\end{figure}
	\subsubsection{Implementation Results}
	\begin{table}[ht]
	\centering
	\caption{Action in the intersection environment. }
    \begin{tabular}{ccc}
    \hline
    Outputs                & Notions               & Limits                                   \\ \hline
    Desired Steering Angle & $\delta_{\text{des}}$ & {[}$-$229$^\circ$, 229$^\circ${]}  $^\dag$      \\
    Desired Acceleration   & $a_{\text{des}}$      & {[}$-$3.0m/s$^\text{2}$, 1.5m/s$^\text{2}${]} \\ \hline
    \end{tabular}\\
     \vspace{3pt}$^\dag$ Negative steering angle means turning left.
    \end{table}
	\begin{figure*}[ht]
	    \subfigure[$7.0$s, $\delta_{\text{des}}=-49.9^\circ$,\\$a_{\text{des}}=-1.18\text{m/s}^{\text 2}$]{\includegraphics[width=0.16\linewidth]{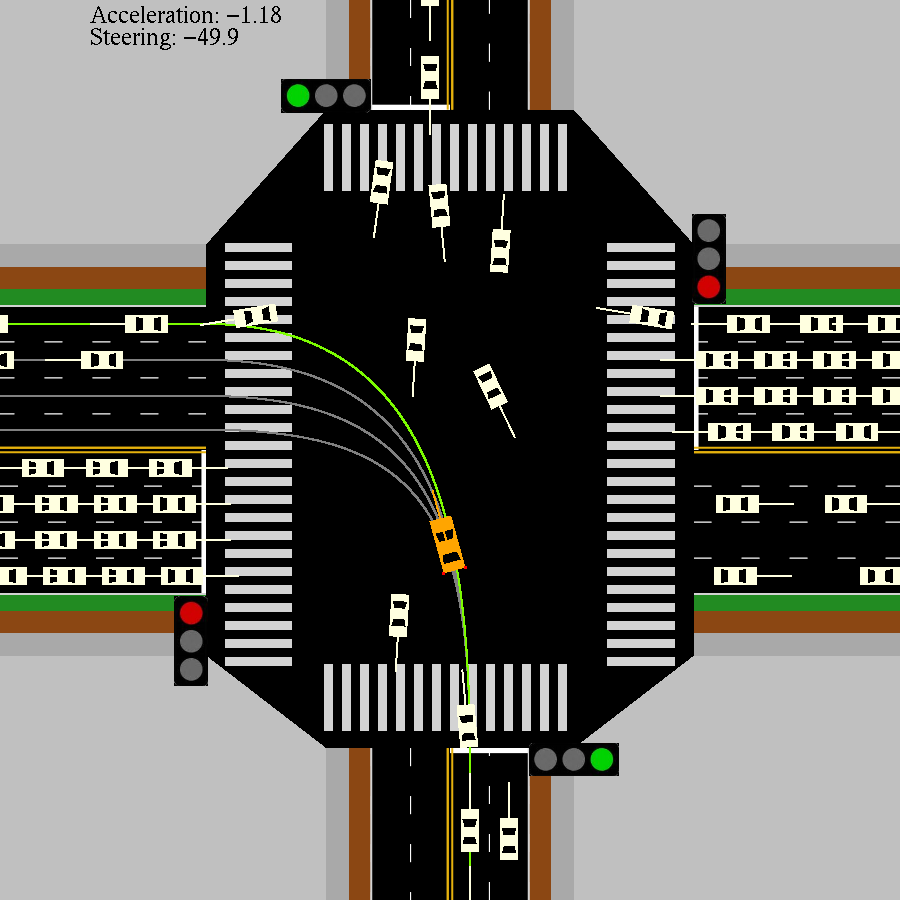}}
	    \subfigure[$7.5$s, $\delta_{\text{des}}=-125^\circ$,\\$a_{\text{des}}=1.49\text{m/s}^{\text 2}$]{\includegraphics[width=0.16\linewidth]{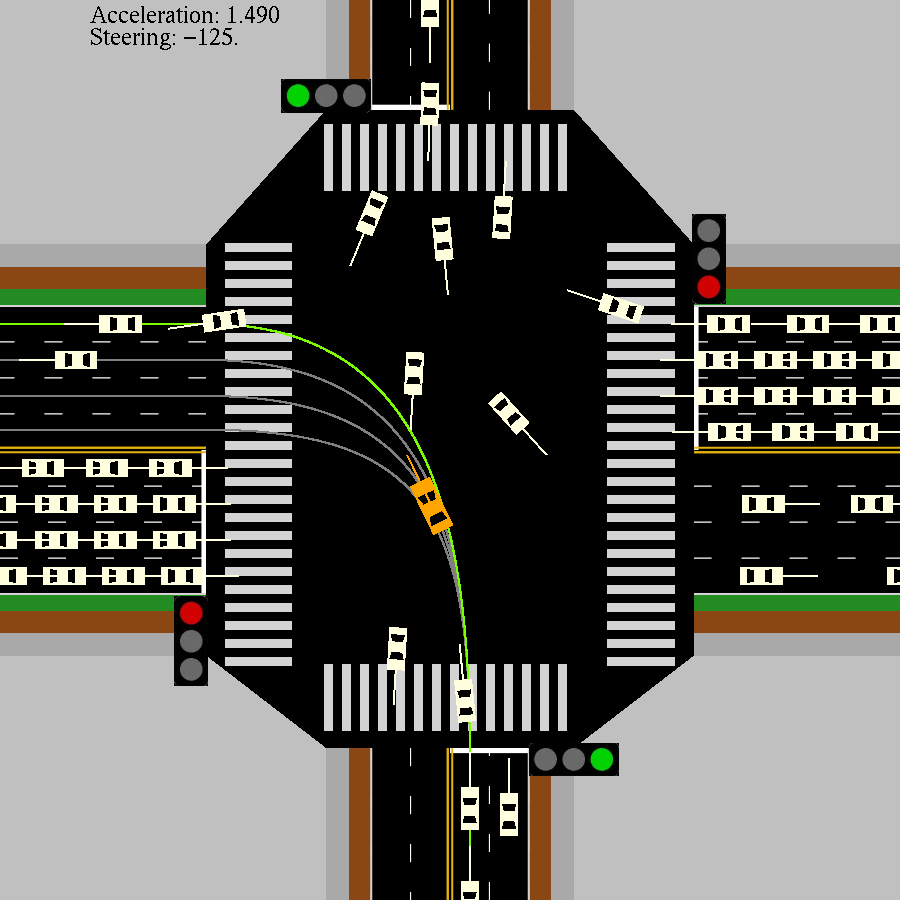}}
	    \subfigure[$8.0$s, $\delta_{\text{des}}=-166^\circ$,\\$a_{\text{des}}=1.33\text{m/s}^{\text 2}$]{\includegraphics[width=0.16\linewidth]{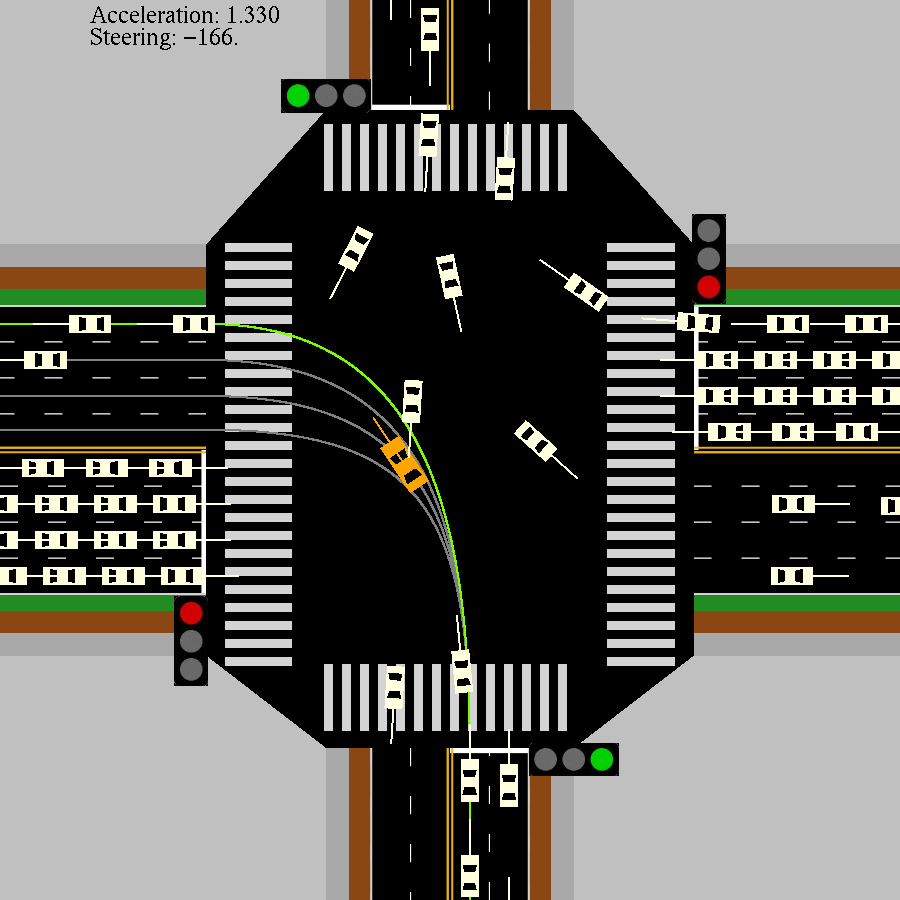}}
	    \subfigure[$8.5$s, $\delta_{\text{des}}=50.7^\circ$,\\$a_{\text{des}}=-2.47\text{m/s}^{\text 2}$]{\includegraphics[width=0.16\linewidth]{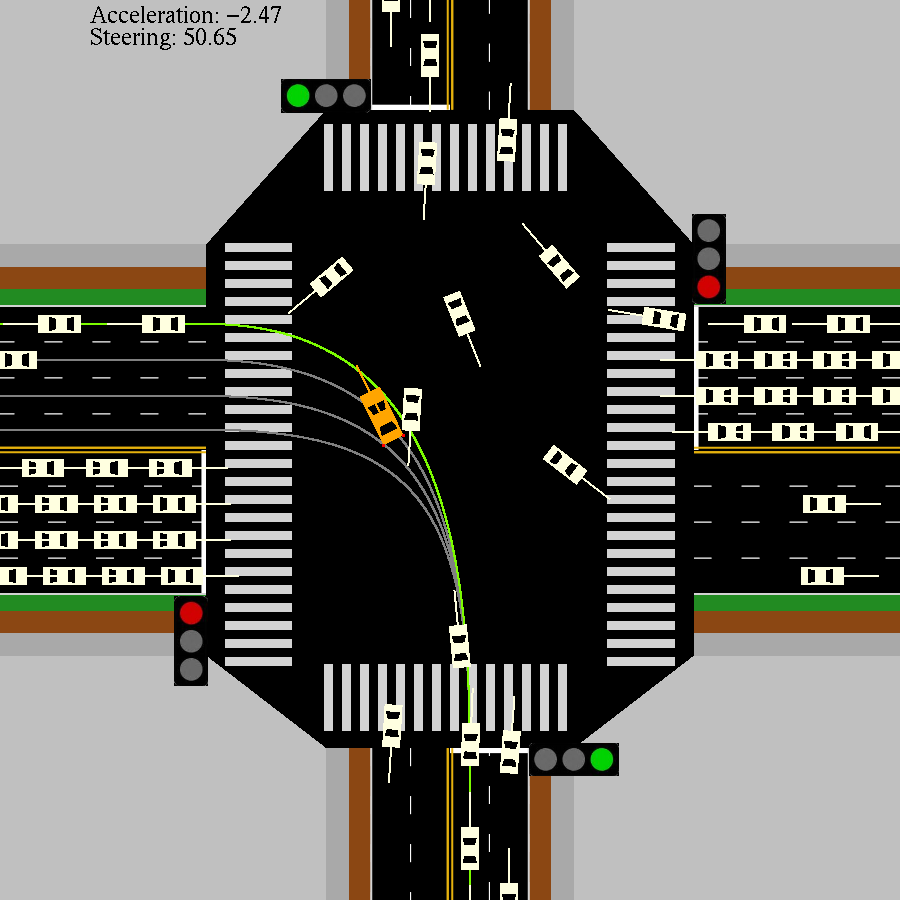}}
	    \subfigure[$9.0$s, $\delta_{\text{des}}=53.0^\circ$,\\$a_{\text{des}}=-2.95\text{m/s}^{\text 2}$]{\includegraphics[width=0.16\linewidth]{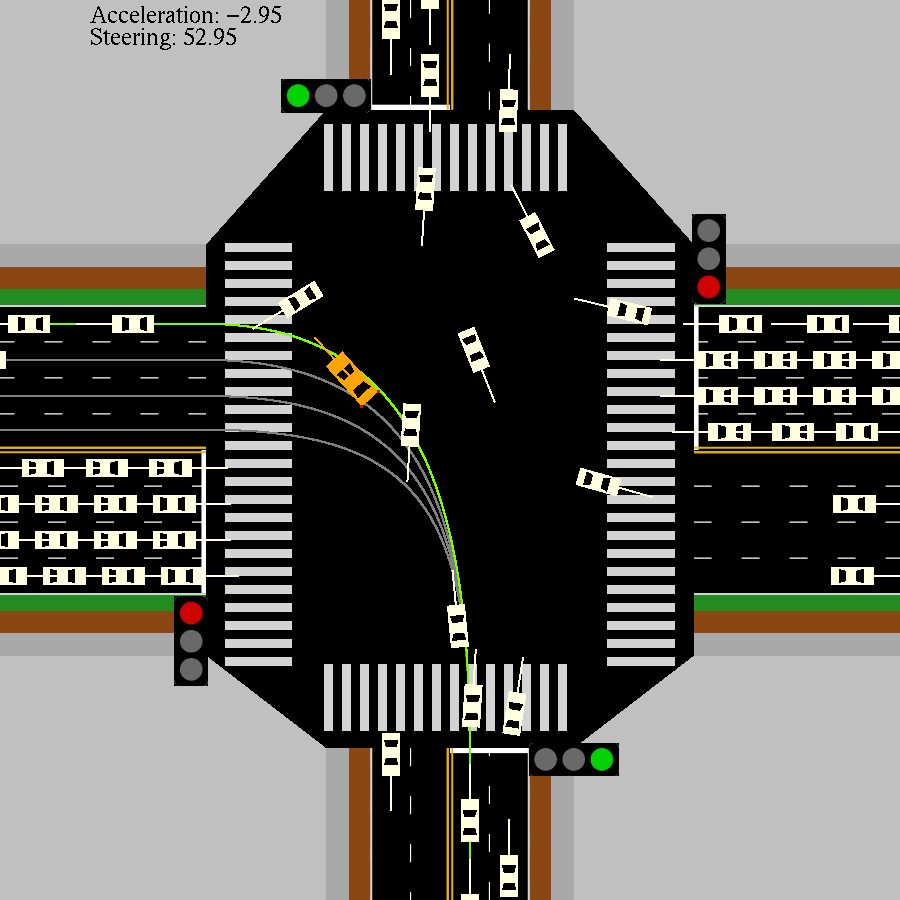}}
	    \subfigure[$9.5$s, $\delta_{\text{des}}=15.4^\circ$,\\$a_{\text{des}}=-2.99\text{m/s}^{\text 2}$]{\includegraphics[width=0.16\linewidth]{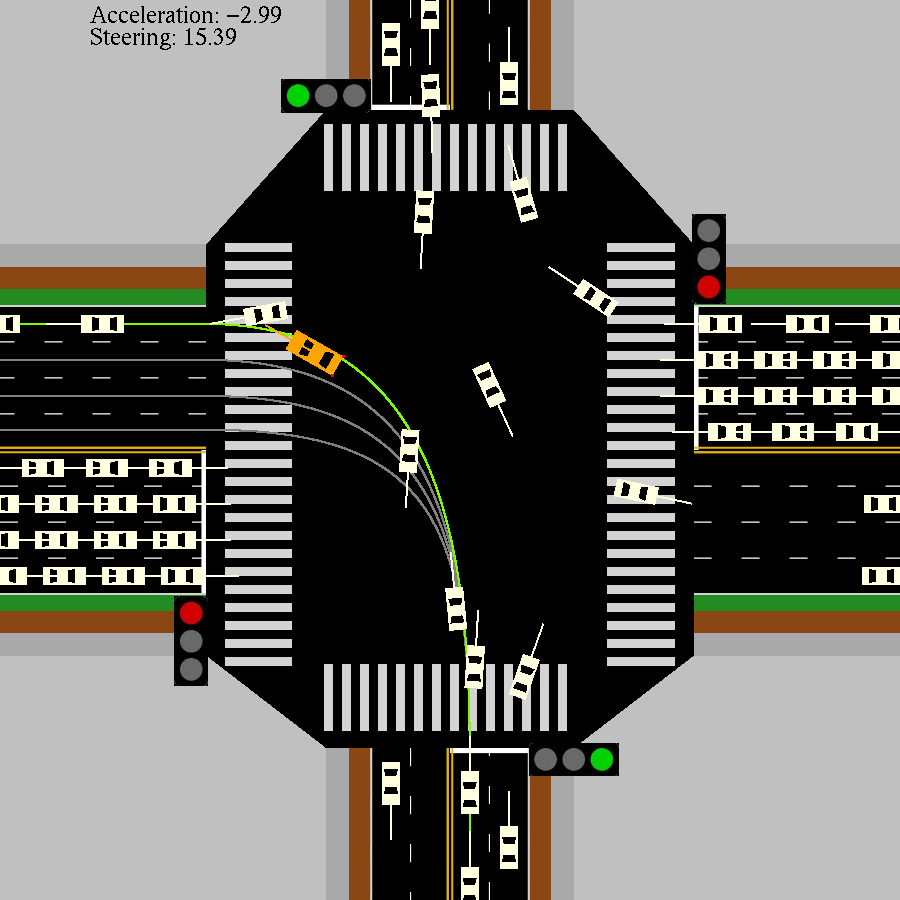}}
	    \caption{Typical collision avoidance maneuvers using SSAC: \emph{left bypassing}. The ego vehicle first brakes a bit, then turn left and accelerate to bypass the opposite vehicle. The ego vehicle then brakes hard to avoid the next vehicle turning right in front.}
	    \label{fig:microscopic1}
	    \subfigure[$6.5$s, $\delta_{\text{des}}=-11.0^\circ$,\\$a_{\text{des}}=0.48\text{m/s}^{\text 2}$]{\includegraphics[width=0.16\linewidth]{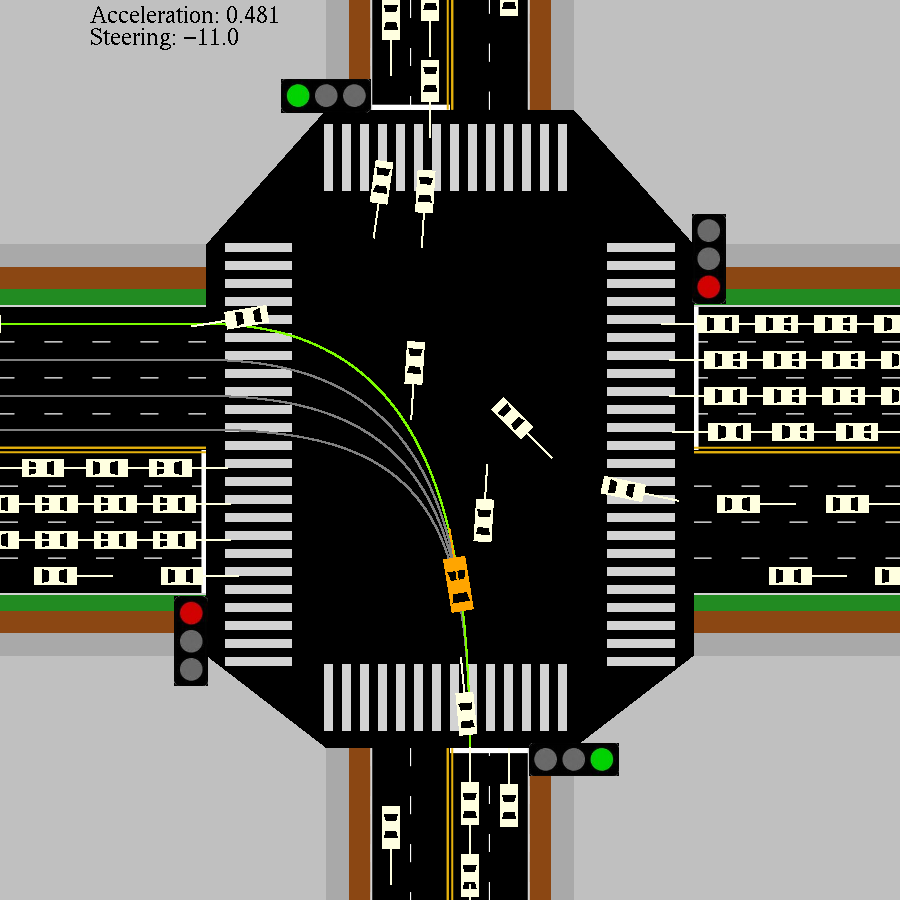}}
	    \subfigure[$7.0$s, $\delta_{\text{des}}=46.6^\circ$,\\$a_{\text{des}}=-2.17\text{m/s}^{\text 2}$]{\includegraphics[width=0.16\linewidth]{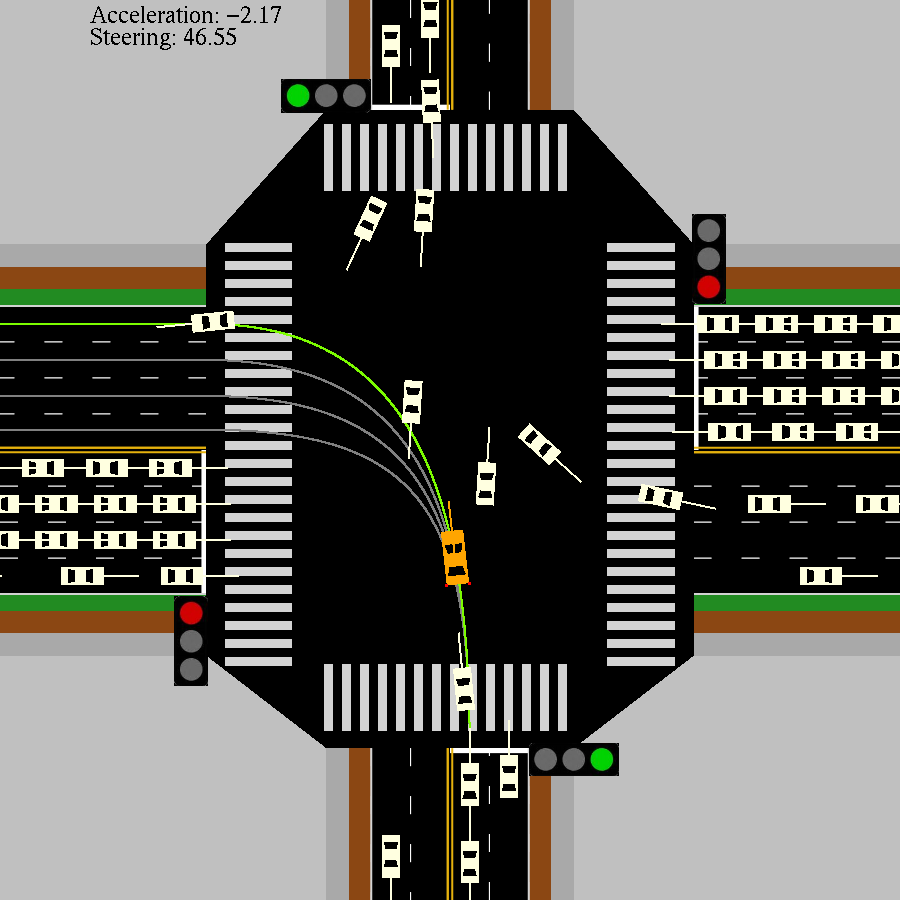}}
	    \subfigure[$7.5$s, $\delta_{\text{des}}=129^\circ$,\\$a_{\text{des}}=0.92\text{m/s}^{\text 2}$]{\includegraphics[width=0.16\linewidth]{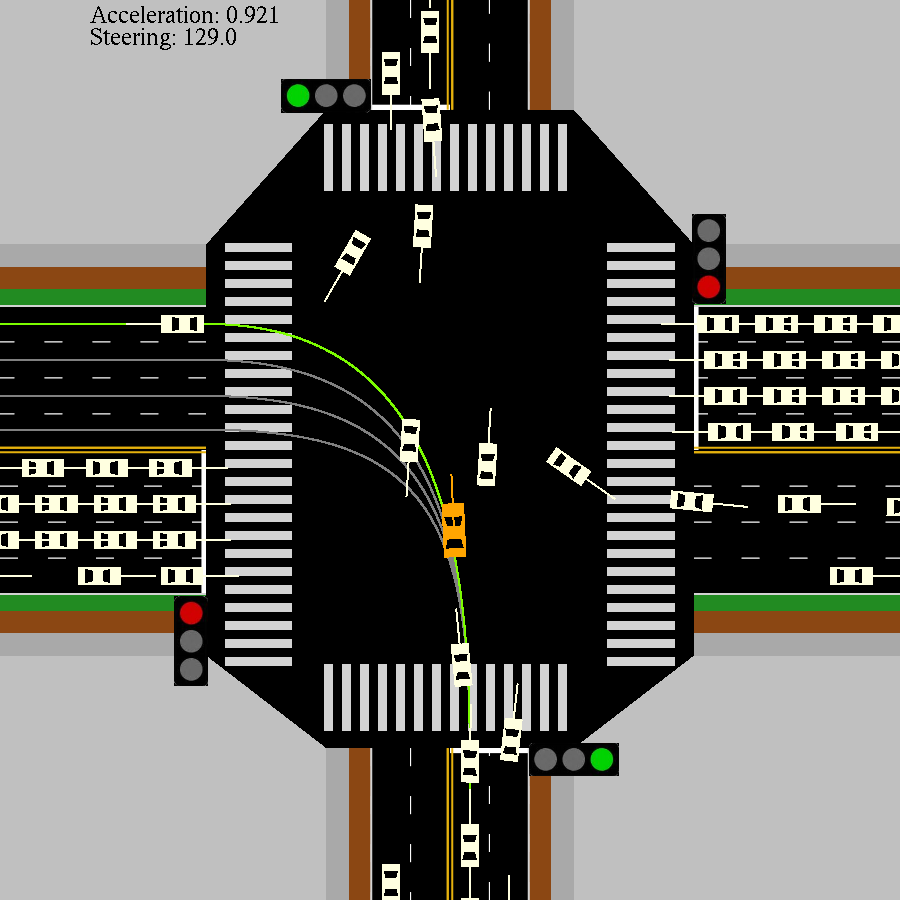}}
	    \subfigure[$8.0$s, $\delta_{\text{des}}=128.9^\circ$,\\$a_{\text{des}}=0.94\text{m/s}^{\text 2}$]{\includegraphics[width=0.16\linewidth]{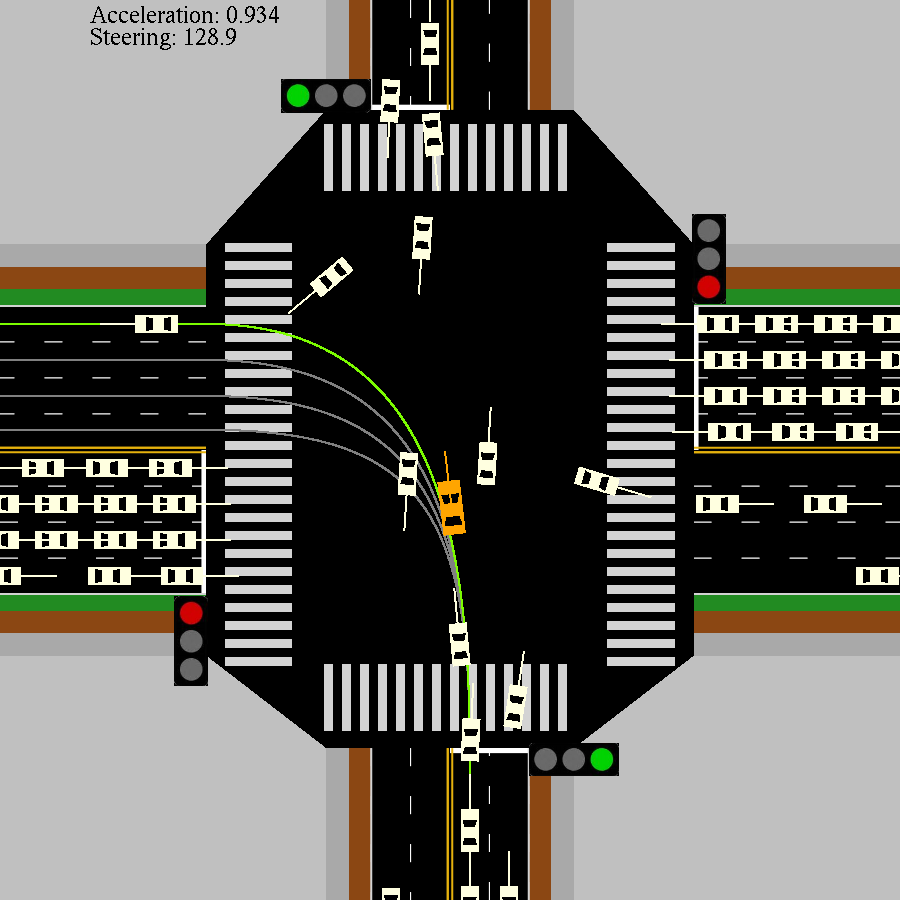}}
	    \subfigure[$8.5$s, $\delta_{\text{des}}=106.5^\circ$,\\$a_{\text{des}}=1.44\text{m/s}^{\text 2}$]{\includegraphics[width=0.16\linewidth]{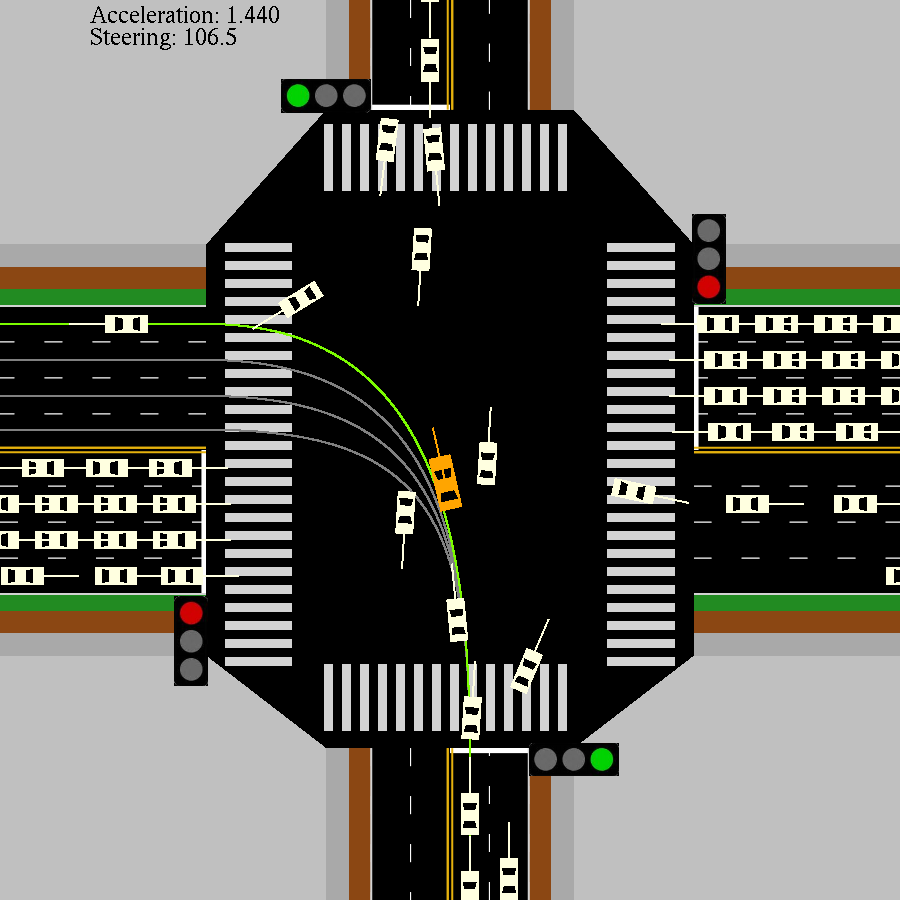}}
	    \subfigure[$9.0$s, $\delta_{\text{des}}=-17.3^\circ$,\\$a_{\text{des}}=-0.64\text{m/s}^{\text 2}$]{\includegraphics[width=0.16\linewidth]{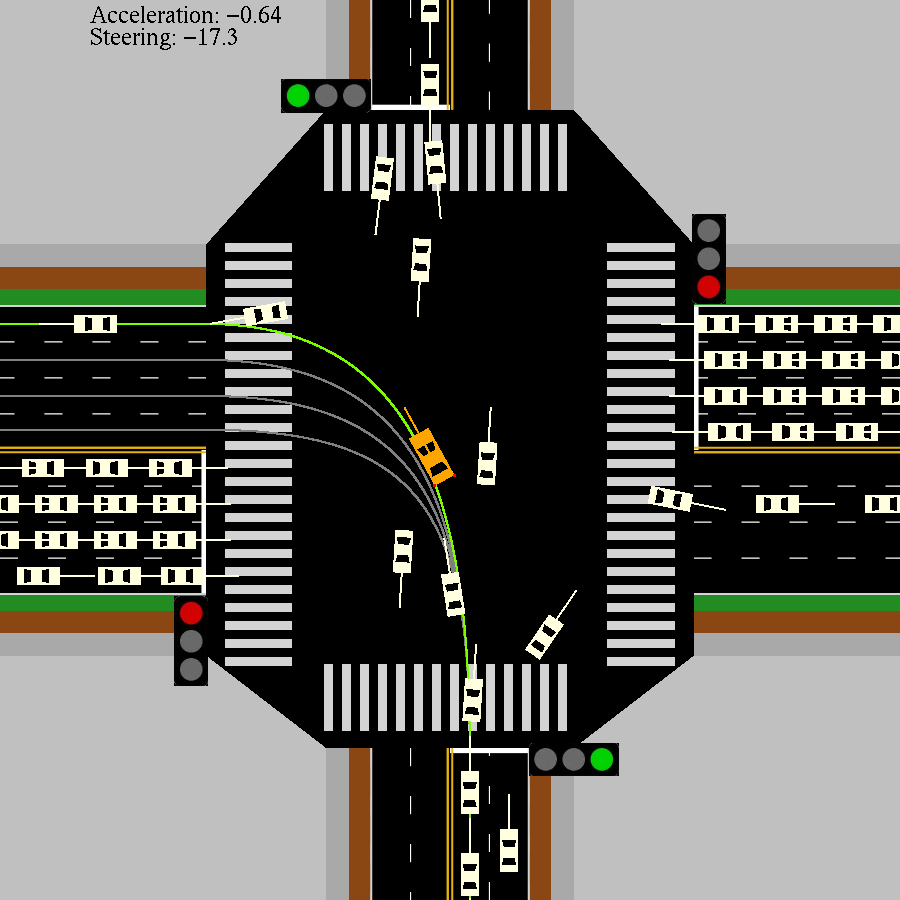}}
	    \caption{Typical collision avoidance maneuvers using SSAC: \emph{right bypassing}. The ego vehicle first brakes a bit, then deviate to the right of the reference trajectory and wait the opposite vehicle to pass. Then ego vehicle accelerate to pass the intersection.}
	    \label{fig:microscopic2}
	    \subfigure[$6.5$s, $\delta_{\text{des}}=-10.4^\circ$,\\$a_{\text{des}}=0.03\text{m/s}^{\text 2}$]{\includegraphics[width=0.16\linewidth]{fig/fig_fail_1/65.png}}
	    \subfigure[$7.0$s, $\delta_{\text{des}}=-14.1^\circ$,\\$a_{\text{des}}=0.82\text{m/s}^{\text 2}$]{\includegraphics[width=0.16\linewidth]{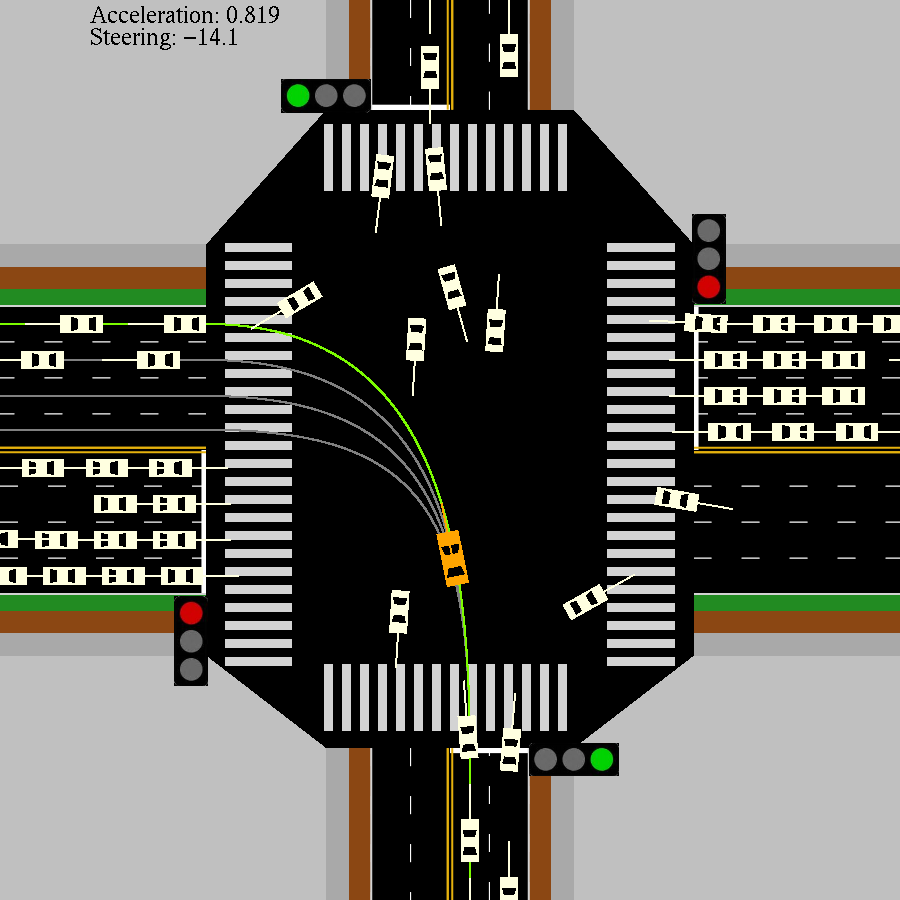}}
	    \subfigure[$7.5$s, $\delta_{\text{des}}=-85.7^\circ$,\\$a_{\text{des}}=1.50\text{m/s}^{\text 2}$]{\includegraphics[width=0.16\linewidth]{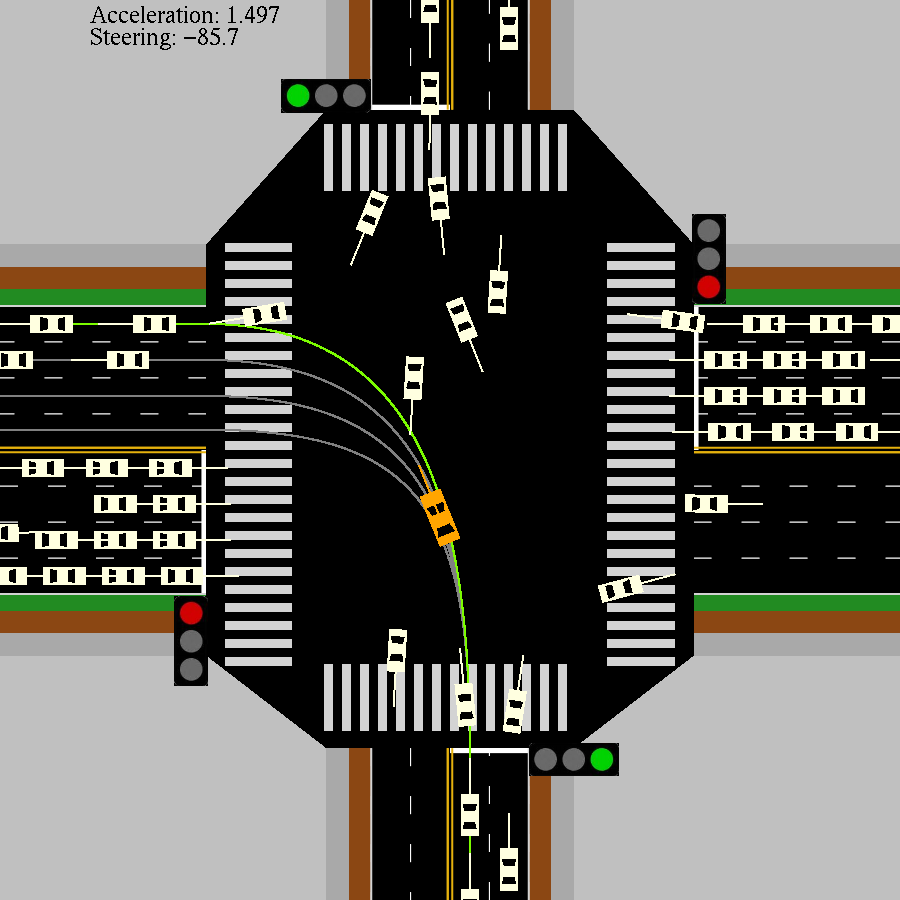}}
	    \subfigure[$8.0$s, $\delta_{\text{des}}=-143^\circ$,\\$a_{\text{des}}=1.38\text{m/s}^{\text 2}$]{\includegraphics[width=0.16\linewidth]{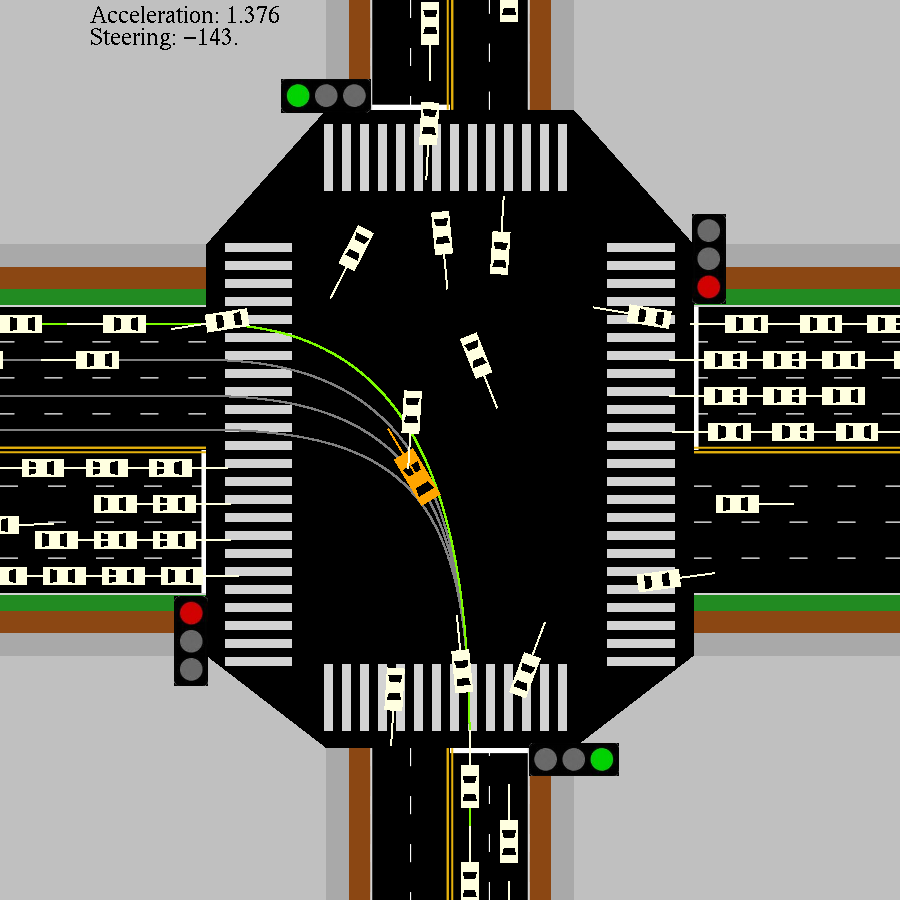}}
	    \subfigure[$8.5$s, $\delta_{\text{des}}=-31.5^\circ$,\\$a_{\text{des}}=-0.10\text{m/s}^{\text 2}$]{\includegraphics[width=0.16\linewidth]{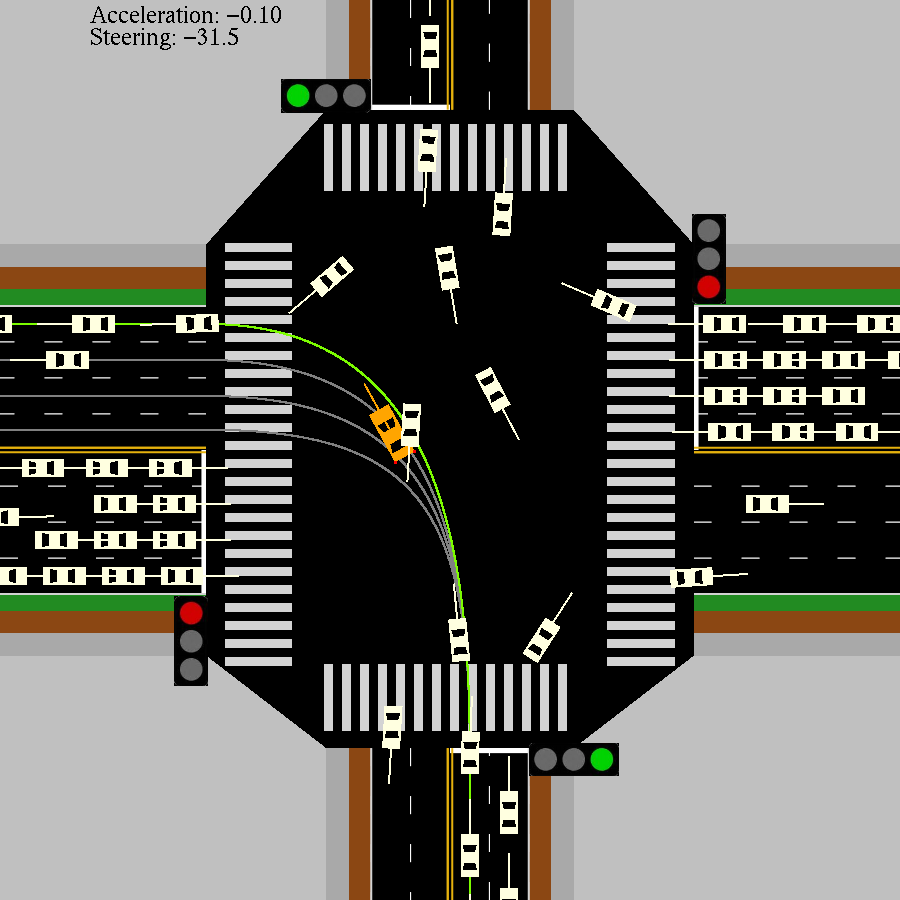}}
	    \subfigure[$9.0$s, $\delta_{\text{des}}=-42.2^\circ$,\\$a_{\text{des}}=-2.84\text{m/s}^{\text 2}$]{\includegraphics[width=0.16\linewidth]{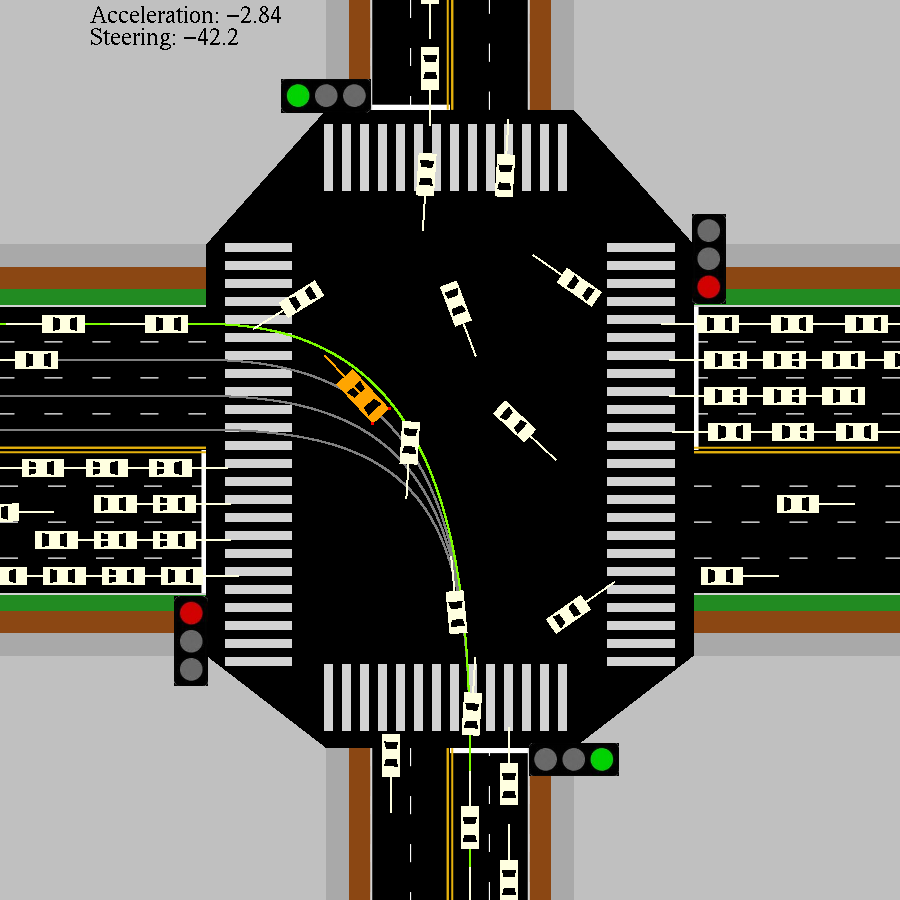}}
	    \caption{Typical collision avoidance maneuvers using baseline algorithm (CPO): \emph{collision}. The ego vehicle does not brake in advance, and turn left later than SSAC, which causes that the ego vehicle collides with the vehicle from the opposite direction.}
	    \label{fig:microscopic3}
	\end{figure*}
 	After training, we implemented the algorithm on the controller with the hardware compatible inference module. We choose the left turn, the task with the most potential conflicts with surrounding vehicles, to compare the differences in the collision avoidance maneuvers between proposed and baseline algorithms. We select a typical case when ego vehicle has to avoid a fast vehicle from the opposite straight line. We capture the screen every 500 ms, and select the conflict periods as shown in Figure \ref{fig:microscopic1}, \ref{fig:microscopic2} and \ref{fig:microscopic3}. The first two episode shown in Figure \ref{fig:microscopic1} and \ref{fig:microscopic2} demonstrate two different avoidance maneuver of SSAC policy. We add reference trajectory (the highlighted one) to compare how much the ego vehicle deviates from it under different algorithms. SSAC in Figure \ref{fig:microscopic1} firstly breaks a bit, and then accelerate pass from below (or left from the perspective of ego vehicle). Ego vehicle keep a small but distinguishable distance from the opposite vehicle. SSAC in Figure \ref{fig:microscopic2} brakes harder and deviates to the right of reference trajectory to wait the opposite vehicle to pass, and then continue to pass the intersection. As for baseline algorithm like CPO in Figure \ref{fig:microscopic3}, it does not choose the wait maneuver since it hurries to pass the intersection. When passing for below, CPO policy does not take enough steering or decelerate so it collide with the opposite vehicle.  CPO turns smaller than SSAC and does not decelerate, which ends in collision with the red vehicle from the opposite side. The value function
 	  	\begin{figure}[h]
 		\centering
 		\includegraphics[width=\linewidth]{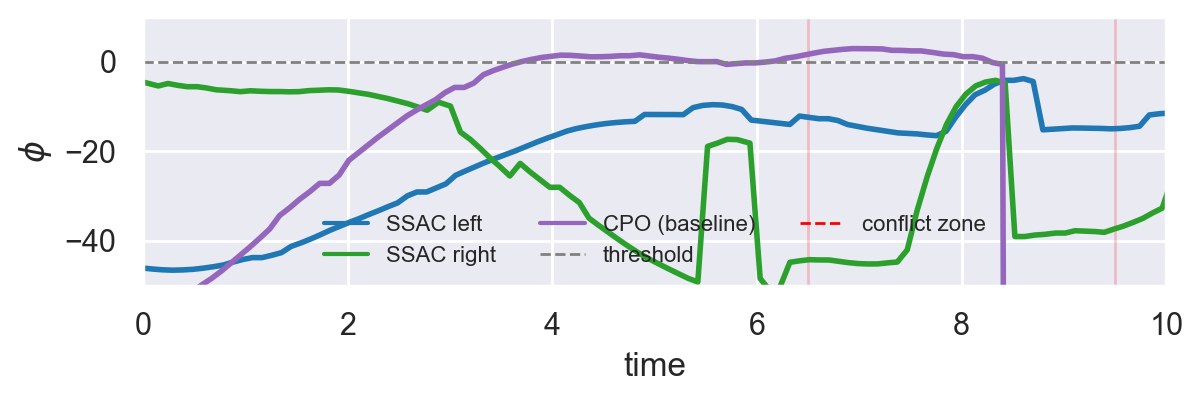}
 		\caption{Safety index $\phi(s)$ value with three corresponding cases. The curve indicates the highest value function among all the filtered surrounding vehicles. The sudden decreasing of safety index is because that the vehicles on the right side of ego vehicle are not considered in the state representation. }
 		\label{fig:valuephi}
 	\end{figure}
 	
 	For the purpose of verifying the constraint-satisfying property of control safe set conditions \eqref{eq:cstrphi0}, we first select 4 random episodes to see the safety index value shown in Figure \ref{fig:valuephi}. As the maximum distance of each random task is different, we regularize the safety index value to demonstrate the trends in the whole episode. It is obvious that $\phi(s)$ is effectively confined when it reaches closely to 0. It verifies the satisfaction the safe set safety constraint \eqref{eq:cstrphi0}.
 \begin{table}%[tbhp]
\centering
\caption{Noise level and the corresponding standard deviation.}
\label{tab.noise_level}
\begin{tabular}{cccccc}
\toprule
Noise level & 0 & 1 & 2 & 3 & 4  \\
\midrule
$\delta_p$ [m]&0 & 0.017 & 0.033 & 0.051 & 0.068  \\
$\delta_{\phi}$ [$\degree$]& 0 & 0.017 & 0.033 & 0.051 & 0.068  \\
$p^j_x,p^j_y$ [m]& 0 & 0.05 & 0.10 & 0.15 & 0.20  \\
$v^j$ [m/s]& 0 & 0.05 & 0.10 & 0.15 & 0.20 \\
$\eta^j$ [$\degree$]& 0 & 1.4 & 2.8 & 4.2 & 5.6 \\
\bottomrule
\end{tabular}
\end{table}
 	Considering real-world sensor noise, we further add white noise term on the observation passed to the controllers, and TABLE \ref{tab.noise_level} depicts the noise level. Figure \ref{fig:boxphi}(b) depicts the distribution of delta safe index, i.e., the LHS of \eqref{eq:cstrphi0}. As the noise level increases, the safety index changes more violently, but the max value maintains as zero indicating that the constraints are not violated. It shows the great robustness of the proposed algorithms.
 	\section{Conclusion}
 	\label{sec:7}
 	In this paper, we proposed the safe set actor critic (SSAC), which can learn a zero constraint violation policy using model-free constrained reinforcement learning. SSAC did not requires specific tasks or type of models, which has the potential to use in general real-world safety-critical artificial intelligence systems. Furthermore, we proved the convergence of the Lagrangian-based constrained RL algorithm with a multiplier neural network, which is also promising for other state-dependent safety constraints in RL problems.
 	
 	Indeed, there are some limitations with the proposed algorithm. For instance, there are still constraint violation in the initial training stage. Our future interest includes investigation of effective pre-training methods to learn the change of safety index before the real-world interaction. The parameterization of the safety index can also be further improved, for example, a neural safety index.
	\appendix
 	\subsection{Proof of Theorem 1}
 	\label{appendix:proof1}
	\begin{step}[Convergence of $\theta$ update]
	\end{step}
		The $\theta$ update can be also formulated as using the stochastic gradient descent:
		\begin{equation}
			\theta_k = -\left.{\hat \nabla}_{\theta} L(\theta,\xi)\right|_{\theta=\theta_{k}}+\delta \theta_{k+1}\label{eq: thetaup}
		\end{equation}
		
		Consider the gradient with respect to $\theta$ at $k$ step is as a random variable $G_{\theta_k}(s_t, a_t)$. First we give the lemma about 
	\label{sec:52}
	\begin{lemma}{Lipschitz continuous of gradient respect to $\theta$.}
		\begin{equation}
			\begin{aligned}
				&{\hat \nabla}_{\theta} L(\theta,\xi) = G_{\theta_k}(s_t, a_t) = \nabla_{\theta} \alpha \log \left(\pi_{\theta}\left(a_{t} \mid s_{t}\right)\right)+\\
				&\left(\nabla_{a_{t}} \alpha \log \left(\pi_{\theta}\left(a_{t} \mid s_{t}\right)\right)-\nabla_{a_{t}}\left(Q_{w}\left(s_{t}, a_{t}\right)-\lambda_{\xi}\left(s_{t}\right) J_{w_c}^{\pi_\theta}\left(s_{t}, a_{t}\right)\right)\right)\\
				&\nabla_{\theta} f_{\theta}\left(\epsilon_{t} ; s_{t}\right)
			\end{aligned}
		\end{equation}
	\end{lemma}
	\begin{proof}
		As the nerual network policy and value is continuously differentiable, and the squashed Gaussian function is constructed by two steps, i.e., compute the action by a normal distribution $u\sim N(\mu(\theta),\sigma^2(\theta))$, and squashing the action by $\tanh$. As both the probability density of normal distribution and $\tanh$ is Lipschitz, The $G_{\theta_k}(s_t, a_t)$ or ${\hat \nabla}_{\theta} L(\theta,\xi)$ is also Lipschitz.
	\end{proof}
	\begin{equation}
		\theta_k = -\left.{\hat \nabla}_{\theta} L(\theta,\xi)\right|_{\theta=\theta_{k}}+\delta \theta_{k+1}\label{eq: thetaup}
	\end{equation}
	and the Martingale difference term with respect to $\theta$ is computed by
	\begin{equation}
		\begin{aligned}
			\delta \theta_{k+1} = {\hat \nabla}_{\theta} L(\theta,\xi)|_{\theta=\theta_{k}} - \mathbb{E}_{s\sim d^{\pi_{\theta}}_\gamma}\big\{\mathbb{E}_{a\sim\pi_\theta} G_{\theta_k}(s_t, a_t)\big\}
		\end{aligned}
	\end{equation}
	The state-action pair used to estimate $\nabla_\theta L$ is sampled from buffer $\mathcal{D}$ which use policy $\pi_\theta$ to collect, so $\mathbb{E}\left\{\delta \theta_{k+1} \mid \mathcal{F}_{\theta, k}\right\}=0$, which means $\{\theta_{k+1}\}$ is a Martingale difference sequence. Besides, the $\theta$-update in \eqref{eq: thetaup} is a stochastic approximation of ODE:
	\begin{equation}
		\dot \theta = - \nabla_\theta L(\theta, \xi)|_{\theta = \theta_k}
		\label{eq:odetheta}
	\end{equation}
	Now we prove that the Martingale difference is square integrable.
	\begin{equation}
		\begin{aligned}
			&\mathbb{E}\left\{\left\|\delta \theta_{k+1}\right\|^{2} \mid \mathcal{F}_{\theta, k}\right\}\\ \leq &
			2 \left\|d^{\pi_\theta}_\gamma(s)\pi(a|s)\right\|_\infty^2\left(\left\|G_{\theta_k}(s, a)\right\|_\infty^2+|G_{\theta_k}(s_t,a_t)|^2\right)\\
			\leq & 6 \left\|d^{\pi_\theta}_\gamma(s)\pi(a|s)\right\|_\infty^2\left\|G_{\theta_k}(s, a)\right\|_\infty^2
		\end{aligned}
		\label{eq: sequaretheta}
	\end{equation}
	Combining all the assumptions and derivations from \eqref{eq: thetaup} and \eqref{eq: sequaretheta} and invoking Theorem 2 in Chapter 2 of Borkar's book \cite{borkar2009stochastic} (stochastic approximation theory for continuous dynamic systems), the sequence $\{\theta_k\}$ converges almost surely to a fixed point $\theta^*$ for ODE \eqref{eq:odetheta}.
	
	Then we show that the fixed point $\theta^*$ is a stationary point using Lyapunov analysis.
	\begin{proposition}
		consider a Lyapunov function for dynamic system \eqref{eq:odetheta}: 
		\begin{equation}
			\mathcal{L}_{\xi}(\theta)=L\left(\theta, \xi\right)-L\left( \theta^{*}, \xi\right)
		\end{equation}
		where $\theta\in\Theta$ is a local minimum. 
		\begin{equation}
			d\mathcal{L}_{\xi}(\theta)/dt\leq0
		\end{equation}
	\end{proposition}
	\begin{proof}
		As we assume that $\theta \in \Theta$ which is a compact set, if $\theta$ lies on the boundary and gradient direction points out of $\Theta$, the update should project $\theta$ into $\Theta$. Then we seperate the proof into two cases:
		
		\textit{Case 1:} $\theta\notin \partial\Theta$ or $\theta\in\partial\Theta$, and the gradient  points inside $\Theta$. Then we have 
		\begin{equation}
			\frac{d L(\theta, \xi)}{d t}=-\left\|\nabla_{\theta} L(\theta, \xi)\right\|^{2} \leq 0 \label{eq:descent}
		\end{equation}
		The equality holds only when $\nabla_{\theta} L(\theta, \xi) =0$.
		
		\textit{Case 2:} $\theta\in\partial\Theta$, and the gradient points outside $\Theta$, and the time derivative of Lyapunov function is 
		\begin{equation}
			\frac{d L(\theta, \xi)}{d t}=\nabla_{\theta} L(\theta, \xi)^T \lim_{\eta\to0} \frac{\Gamma_\Theta\left\{\theta - \nabla_\theta\eta L(\theta, \xi)\right\} - \theta}{\eta} \label{eq:descent}
		\end{equation}
		where $\Gamma_\Theta\{\cdot\}$ is to project a vector to $\theta$. Denote $\theta_\eta = \theta - \eta\nabla_\theta L(\theta, \xi)$ for a positive $\eta$. According to the proposition 2.1.3 in \cite{bertsekas1997nonlinear}
		\begin{equation}
			(\theta_\eta - \Gamma_\Theta\left\{\theta_\eta\right\})^T(\theta - \Gamma_\Theta\left\{\theta_\eta\right\}) \leq 0,\quad\forall\theta\in\Theta
		\end{equation}
		which further implies that
		\begin{equation}
			\begin{aligned}
				&(\theta -\theta_\eta )^T( \Gamma_\Theta\left\{\theta_\eta\right\} - \theta)\\
				=&(\theta_\eta - \Gamma_\Theta\left\{\theta_\eta\right\} + \Gamma_\Theta\left\{\theta_\eta\right\} - \theta)^T(\theta - \Gamma_\Theta\left\{\theta_\eta\right\}) \\
				=&-\|\theta - \Gamma_\Theta\left\{\theta_\eta\right\}\|^2 + (\theta_\eta - \Gamma_\Theta\left\{\theta_\eta\right\})^T(\theta - \Gamma_\Theta\left\{\theta_\eta\right\}) \leq 0
			\end{aligned}
		\end{equation}
		Then the time derivative is
		\begin{equation}
			\begin{aligned}
					\frac{d L(\theta, \xi)}{d t}= & \frac{(\theta - \theta_\eta)^T}{\eta}\lim_{\eta\to0}\frac{\Gamma_\Theta\left\{\theta_\eta\right\} - \theta}{\eta} \\
					= & \frac{1}{\eta} \lim_{\eta\to0} (\theta -\theta_\eta )^T( \Gamma_\Theta\left\{\theta_\eta\right\} - \theta) \leq 0
			\end{aligned}
		\end{equation}
		for positive $\eta$. The equality only holds when the directional derivative term $\|\lim_{\eta\to0} \frac{\Gamma_\Theta\left\{\theta - \nabla_\theta\eta L(\theta, \xi)\right\} - \theta}{\eta}\|=0$. Therefore, $\theta^*$ is a stationary point of dynamic system $\theta(t)$.
	\end{proof}
	Combining the aforementioned conclusions, $\{\theta_k\}$ converges almost surely to a local minimum point $\theta^*$ for given $\xi$.
	\begin{step}[Convergence of $\lambda$ update] 
	\end{step}
	The stochastic gradient with respect to $\xi$ is 
		\begin{equation}
		G_\lambda(s_t, a_t) = \hat{\nabla} L_\xi(\theta^*_{\xi},\xi)= \left(J_{w_c}(s_t, a_t)\right) \nabla_{\xi} \lambda_{\xi}\left(s_{t}\right)
	\end{equation}
	\begin{proposition}
		$\nabla_\xi L(\theta^*_{\xi},\xi)$ is Lipschitz on $\xi$.
	\end{proposition}
	\begin{proof}
		$J_{w_c}(s_t, a_t)$ takes finite values by $\eqref{eq:sis}$. Considering that $\Xi$ is a compact set and the continuous differentiable assumption of neural network, the $\nabla_\xi L(\theta^*_{\xi},\xi)$ is Lipschitz.
	\end{proof}
	Similarly, the $\xi$ update can also be formulated as:
	\begin{equation}
		\xi_{k+1} = {\hat \nabla}_{\xi} L(\theta,\xi)|_{\theta=\theta_k,\xi=\xi_{k}} + \delta\xi_{k+1} +\delta\theta_{\epsilon}
	\end{equation}
	where $\theta_k = \theta^*_{\xi} + \epsilon_\theta$, $\epsilon_\theta$ is an estimation error in $\theta$ update.
	Similarly, the error term
	\begin{equation}
		\begin{aligned}
			\delta \xi_{k+1} = {\hat \nabla}_{\xi} L(\theta^*(\xi),\xi)|_{\xi=\xi_{k}} - \mathbb{E}_{s\sim d^{\pi_{\theta}}_\gamma}\big\{\mathbb{E}_{a\sim\pi_\theta} G_{\xi_k}(s_t, a_t)\big\}
		\end{aligned}
	\end{equation}
	is square integrable since
	\begin{small}
			\begin{equation}
			\begin{aligned}
				&\mathbb{E}\left\{\left\|\delta \xi_{k+1}\right\|^{2} \mid \mathcal{F}_{\xi, k}\right\}\\ \leq &
				4 \left\|d^{\pi_\theta}_\gamma(s)\pi(a|s)\right\|_\infty^2\left(\max \left\|Q_{w_{C}}^{\pi_\theta^*}\left(s_{t}, a_{t}\right)\right\|^2+d^2\right)\left\|\nabla_\xi\lambda_\xi(s_t)\right\|_\infty^2\\
			\end{aligned}
		\end{equation}
	\end{small}
	The error term cased by $\theta$ estimation error is 
	\begin{equation}
		\begin{aligned}
			\delta \theta_\epsilon = & {\hat \nabla}_{\xi} L(\theta,\xi)|_{\theta=\theta_k, \xi=\xi_{k}} - {\hat \nabla}_{\xi} L(\theta^*(\xi),\xi)|_{\xi=\xi_{k}}\\
			= & \left(Q^{\pi_{\theta_k}}(s_t,a_t) - Q^{\pi_{\theta^*}}(s_t,a_t)\right)\nabla_\xi\lambda_\xi(s_t)\\
			= & \left(\nabla_aQ(s_t, a_t)\nabla_\theta\pi(s_t)\epsilon_{\theta_k} +o(\left\|\epsilon_{\theta_k}\right\|) \right)\nabla_\xi\lambda_\xi(s_t)
		\end{aligned}
	\end{equation}
	Therefore, $\|\delta\theta_\epsilon\|\to 0$ as $\|\epsilon_\theta\|\to0$. 
	Combining the Lipschitz continuity and error property, one can invoke Theorem 2 in Chapter 6 of Borkar's book \cite{borkar2009stochastic} to show that the sequence $\{\xi_k\}$ converges to the solution of following ODE:
	\begin{equation}
		\dot{\xi} = - \nabla_\xi L(\theta, \xi)|_{\theta = \theta^*(\xi)}
	\end{equation}
	
	Then we show that the fixed point $\xi^*$ is a stationary point using Lyapunov analysis. Note that we have to take $\epsilon_\theta$ into considerations.
	\begin{proposition}
		For the dynamic system with error term
		\begin{equation}
			\dot{\xi} = - \nabla_\xi L(\theta, \xi)|_{\theta = \theta^*(\xi)+\epsilon_{\theta}}
		\end{equation}
		Define a Lyapunov function to be
		\begin{equation}
			\mathcal{L}(\xi)=L\left(\theta^*, \xi\right)-L\left( \theta^{*}, \xi^{*}\right)
		\end{equation}
		where $\xi^*$ is a local maximum point. Then $\frac{d\mathcal{L}(\xi)}{dt}\leq 0$.
	\end{proposition}
	\begin{proof}
		The proof is similar to Proposition 2, only the error of $\theta$ should be considered. We prove that the error of $\theta$ does not affect the decreasing property here:
		\begin{small}
			\begin{equation}
				\begin{aligned}
					&\frac{d\mathcal{L}(\xi)}{dt} =  - \left(\nabla_\xi L(\theta, \xi)|_{\theta = \theta^*(\xi)+\epsilon_{\theta}} \right)^T \nabla_\xi L(\theta, \xi)|_{\theta = \theta^*(\xi)} \\
					= & - \left\|\nabla_\xi L(\theta, \xi)|_{\theta = \theta^*(\xi)}\right\|^2 - \delta\theta_\epsilon^T \nabla_\xi L(\theta, \xi)|_{\theta = \theta^*(\xi)}\\
					%			\leq & - \left\|\nabla_\xi L(\theta, \xi)|_{\theta = \theta^*(\xi)}\right\|^2 + 
					%			\left\|\delta\theta_\epsilon\right\| \left\|\nabla_\xi L(\theta, \xi)|_{\theta = \theta^*(\xi)}\right\|\\
					\leq & - \left\|\nabla_\xi L(\theta, \xi)|_{\theta = \theta^*(\xi)}\right\|^2 + K_1 \left\|\epsilon_\theta\right\|\left\|\nabla_\xi L(\theta, \xi)|_{\theta = \theta^*(\xi)}\right\|\\
					= & \left(- K_2\left\|\epsilon_\xi\right\| + o(\left\|\epsilon_\xi\right\|) + K_1 \left\|\epsilon_\theta\right\| \right) \left\|\nabla_\xi L(\theta, \xi)|_{\theta = \theta^*(\xi)}\right\|
				\end{aligned}
			\end{equation}
		\end{small}
		As $\theta$ converges faster than $\xi$, $d\mathcal{L}_{\zeta}(\xi)/dt\leq0$, so there exists trajectory $\xi(t)$ converges to $\xi^*$ if initial state $\xi_0$ starts from a ball $\mathcal{B}_{\xi^*}$ around $\xi^*$ according to the asymptotically stable systems.
	\end{proof}
	\begin{table}[htp]
	\vskip 0.15in
	\begin{center}
		\caption{Detailed hyperparameters.}
		\label{table.hyper}
		\begin{tabular}{lc}
			\toprule
			Algorithm & Value \\
			\hline
			\emph{SSAC} & \\
			\quad Optimizer &  Adam ($\beta_{1}=0.9, \beta_{2}=0.999$)\\
			\quad Approximation function  &Multi-layer Perceptron \\
			\quad Number of hidden layers & 2\\
			\quad Number of hidden units & 256\\
			\quad Nonlinearity of hidden layer& ELU\\
			\quad Nonlinearity of output layer& linear\\
			\quad Actor learning rate & Linear annealing $3{\rm{e-}}5\rightarrow1{\rm{e-}}6 $\\
			\quad Critic learning rate & Linear annealing $8{\rm{e-}}5\rightarrow1{\rm{e-}}6 $\\
			\quad  Learning rate of multiplier net & Linear annealing $5{\rm{e-}}5\rightarrow5{\rm{e-}}6 $ \\
			\quad  Learning rate of $\alpha$ & Linear annealing $5{\rm{e-}}5\rightarrow1{\rm{e-}}6 $ \\
			\quad Reward discount factor ($\gamma$) & 0.99\\
			\quad Policy update interval ($m_\pi$) & 3\\
			\quad  Multiplier ascent interval ($m_\lambda$)& 12\\
			\quad Target smoothing coefficient ($\tau$) & 0.005\\
			\quad Max episode length ($N$) & \\
			\quad\quad Safe exploration tasks & 1000\\
			\quad\quad Autonomous driving task & 200\\
			\quad  Expected entropy ($\overline{\mathcal{H}}$) &  $\overline{\mathcal{H}}=-$Action Dimentions \\
			\quad  Replay buffer size & $5\times10^5$\\
			\quad  Replay batch size & 256\\
			\quad Safety index hyperparameters  & \\
			\quad $(\eta, n, k, \sigma)$ & $(0,\ 2,\ 1,\ 0.04)$\\\midrule
			\emph{CPO, TRPO-Lagrangian} &\\ 
			\quad Max KL divergence&  $0.1$\\
			\quad Damping coefficient&  $0.1$\\
			\quad Backtrack coefficient&  $0.8$\\
			\quad Backtrack iterations&  $10$\\
			\quad Iteration for training values&  $80$\\
			\quad Init $\lambda$ &  $0.268 (softplus(0))$\\
			\quad GAE parameters  &  $0.95$\\
			\quad Batch size &  $2048$\\
			\quad Max conjugate gradient iters & $ 10$ \\
			
			\hline
			\emph{PPO-Lagrangian} &\\ 
			\quad Clip ratio &  $0.2$\\
			\quad KL margin &  $1.2$\\
			\quad Mini Bactch Size & $64$\\
			\bottomrule
		\end{tabular}
	\end{center}
	\vskip -0.1in
\end{table}
	\begin{step}[Local saddle point $(\theta^*, \xi^*)$]
	\end{step}
	As we provided in previous, $L(\theta^*,\xi^*)\leq L(\theta,\xi^*)$, so we need to prove here $L(\theta^*,\xi^*) \geq L(\theta^*,\xi)$. 
	To complete the proof we need that 
	\begin{equation}
		J^{\pi_{\theta^*}}_{w_c}(s_t,a_t)\leq d \text{ and } \lambda^*(s_t)(J^{\pi_{\theta^*}}_{w_c}(s_t,a_t)-d)=0
	\end{equation}
	for all $s$ in the discounted distribution. Recall that $\lambda^*$ is a local maximum point, we have
	\begin{equation}
		\nabla_\xi L(\theta^*,\xi^*) = 0
	\end{equation}
	Assume there exists $s_t$ so that $J^{\pi_{\theta^*}}_{w_c}(s_t,\pi^*(s_t)) > 0$. Then for $\lambda^*$ we have 
	\begin{equation}
		\nabla_\xi L(\theta^*,\xi^*) = d_\gamma^{\pi_\theta^*}(s_t)J^{\pi_{\theta^*}}_{w_c}(s_t,\pi^*(s_t))\nabla_\xi\lambda_\xi(s_t) \neq 0
	\end{equation}
	The second part only requires that $\lambda^*(s_t)=0$ when $J^{\pi_{\theta^*}}_{w_c}(s_t,\pi^*(s_t)) < 0$. Similarly, we assume that there exists $s_t$ and $\xi^*$ where $\lambda_{\xi^*}(s_t) > 0$ and $Q^{\pi_{\theta^*}}(s_t,\pi^*(s_t)) < d$. 
	there must exists a $\xi_0$ subject to
	\begin{equation}
		\xi_0 = \xi^* + \eta_0 d_\gamma^{\pi_{\theta^*}}(s_t)\left(J^{\pi_{\theta^*}}_{w_c}(s_t,\pi^*(s_t)) - d\right)\nabla_\xi\lambda_\xi(s_t)
	\end{equation}
	for any $\eta\in(0,\eta_0]$ where $\eta_0\geq0$. It contradicts the statement the local maximum $\xi^*$.
	Then we get
	\begin{equation}
		\begin{aligned}
			L(\theta^*,\xi^*) & = J_r(\theta^*) + \mathbb E_{s\sim d_\gamma}\left\{\lambda_{\xi^*}(s)\mathbb{E}_{a\sim\pi}\{J^{\pi_{\theta^*}}_{w_c}(s,a)\}\right\} \\
			& = J_r(\theta^*) \\
			&\geq J_r(\theta^*) +  \mathbb E_{s\sim d_\gamma}\left\{\lambda_{\xi^*}(s)\mathbb{E}_{a\sim\pi}\{J^{\pi_{\theta^*}}_{w_c}(s,a)\}\right\} \\
			& = L(\theta^*,\xi)
		\end{aligned}
	\end{equation}
	So $(\theta^*, \xi^*)$ is a locally saddle point.
\subsection{Experimental Details}
\label{sec:appendixexp}
\subsubsection{Implementation Details}
Implementation of FAC are based on the Parallel Asynchronous Buffer-Actor-Learner (PABAL) architecture proposed by \cite{guan2021mixed}. All experiments are implemented on 2.5 GHz Intel Xeon Gold 6248 processors with 12 parallel actors, including 4 workers to sample, 4 buffers to store data and 4 learners to compute gradients.\footnote{%Our open-source implementation of FAC can be found at \url{https://github.com/mahaitongdae/Feasible-Actor-Critic}. 
	The original implementation of PABAL can be found at \url{https://github.com/idthanm/mpg}.  The baseline implementation is modified from \url{https://github.com/openai/safety-starter-agents} and \url{https://github.com/ikostrikov/pytorch-a2c-ppo-acktr-gail}.} The simulation environments runs on a personal computer with 3.0 GHz AMD Ryzen 4900HS processors with 16 GB memory.
\subsubsection{Hyperparameters}
The hyperparameters of FAC and baseline algorithms are listed in Table \ref{table.hyper}.

\bibliographystyle{IEEEtran}
\bibliography{tai.bib}
\end{document}

%% file: def.tex
\newcommand{\E}{\mathbb{E}}
\newcommand{\policy}{\pi}
\newcommand{\piparas}{\theta}
\newcommand{\qparas}{w}
\newcommand{\lamparas}{\xi}
\newcommand{\state}{s}
\newcommand{\st}{{\state_t}}
\newcommand{\stp}{{\state_{t+1}}}
\newcommand{\action}{a}

\newcommand{\at}{{\action_t}}
\newcommand{\reward}{r}

\newcommand{\density}{p}
\newcommand{\pdyn}{\density}